\documentclass{article} 
\usepackage{iclr2027_conference,times}

\usepackage{hyperref}
\usepackage{url}

\title{Learning Unknown Constraints without Unsafe Data via Optimality and Counterfactual Regularization}

\author{%
  Zhouyu Zhang
  \\
  School of Electrical and Computer Engineering \\
  Georgia Institute of Technology \\
  Atlanta, GA 30309 \\
  \texttt{zzhang3097@gatech.edu} \\
  \And
  Chih-Yuan Chiu
  \\
  School of Electrical and Computer Engineering \\
  Georgia Institute of Technology \\
  Atlanta, GA 30309 \\
  \texttt{cyc@gatech.edu} \\
  \And
  Glen Chou \\
  School of Cybersecurity and Privacy \\
  School of Aerospace Engineering \\
  Georgia Institute of Technology \\
  \texttt{chou@gatech.edu}
}

\usepackage[utf8]{inputenc} 
\usepackage[T1]{fontenc} 
\usepackage{url} 
\usepackage{nicefrac}       
\usepackage{microtype}      

\usepackage{amsmath,amssymb,amsfonts}

\usepackage{float}

\usepackage{booktabs}
\usepackage{adjustbox}
\usepackage{wrapfig}
\usepackage{lipsum}
\usepackage{amsthm}
\usepackage{amssymb}
\usepackage{amsmath}
\usepackage{siunitx}
\usepackage{mathtools}
\usepackage{xcolor}
\usepackage{multirow}

\usepackage[inline]{enumitem}
\usepackage[ruled,vlined,linesnumbered,noend]{algorithm2e}
\usepackage[noend]{algpseudocode}
\usepackage{ifthen}
\usepackage{dsfont}
\usepackage{balance}
\usepackage{dirtytalk}
\usepackage{afterpage}
\usepackage{bm}
\usepackage{enumitem}

\usepackage{tabularx}

\usepackage{hyperref} 

\allowdisplaybreaks

\newcommand{\controlSpace}{\mathcal{A}}

\newcommand{\D}{\mathcal{D}}

\newcommand{\E}{\mathbb{E}}

\newcommand{\bigO}{\mathcal{O}}

\newcommand{\Prob}{\mathbb{P}}

\newcommand{\ra}{\rightarrow}

\newcommand{\R}{\mathbb{R}}

\newcommand{\stateSpace}{\mathcal{S}}

\newtheorem{remark}{Remark}

\newtheorem{proposition}{Proposition}

\newtheorem{lemma}{Lemma}

\newcommand{\KKT}{\text{KKT}}

\newcommand{\boldu}{\textbf{u}}
\newcommand{\boldx}{\textbf{x}}

\newcommand{\boldlambda}{\boldsymbol{\lambda}}
\newcommand{\boldnu}{\boldsymbol{\nu}}
\newcommand{\boldeta}{\boldsymbol{\eta}}
\newcommand{\boldrho}{\boldsymbol{\rho}}
\newcommand{\boldkappa}{\boldsymbol{\kappa}}

\newcommand{\bolds}{\textbf{s}}
\newcommand{\boldz}
{{\boldsymbol{z}}}
\newcommand{\boldzero}{\textbf{0}}

\newcommand{\loss}{\mathcal{L}}
\newcommand{\lagrangian}{L}

\newcommand{\invRel}{{\text{inv}, \text{rel}}}

\newcommand{\paren}[1]{{({#1})}}

\iclrfinalcopy
\begin{document}

\maketitle
\lhead{}

\begin{abstract}
\vspace{-3mm}
\looseness-1Learning from demonstrations (LfD) provides a framework for inferring unknown constraints from locally optimal, constraint-satisfying expert behavior. 
Existing approaches largely fall into two paradigms, constrained inverse optimal control (CIOC) and inverse constrained reinforcement learning (ICRL).
CIOC exploits optimality conditions such as the Karush--Kuhn--Tucker (KKT) conditions but typically assumes known dynamics and structured constraint representations.
Meanwhile, ICRL accommodates complex unknown constraints and unknown transition dynamics but often requires extensive online exploration, during which unsafe constraint violations may occur.
In this work, we introduce Counterfactual KKT (CF-KKT), a constraint learning framework that leverages learned dynamics and locally optimal demonstrations to recover unknown constraints without requiring known dynamics or additional risky exploration, thereby combining the data efficiency and safety advantages of CIOC with the flexibility of ICRL. First, we use a locally learned differentiable dynamics model to impose KKT-inspired optimality conditions directly on the demonstrations. Second, we use the learned dynamics to generate reward-improving counterfactual behaviors near the demonstrations, revealing behaviors that would be preferable in the absence of the unknown constraint and thus providing synthetic infeasible data. When the constraint parameterization is known, the same learned-dynamics framework enables direct CIOC-based parameter recovery, and we characterize its sensitivity to dynamics misspecification. Across high-dimensional robotic control tasks, our approach learns neural constraint representations with improved safety and data efficiency relative to state-of-the-art offline ICRL baselines.
\end{abstract}
\vspace{-22pt}
\section{Introduction}
\label{sec: Introduction}
\vspace{-8pt}


Autonomous systems deployed in safety-critical 
applications, such as autonomous driving and warehouse operations,
must adhere to constraints that encode safety and task specifications \citep{Altman1999ConstrainedMDP, Achiam2017ConstrainedPolicyOptimization, Borkar2005AnActorCriticAlgorithmforConstrainedMDPs, Mayne2000ConstrainedMPC, Berberich2021DataDrivenModelPredictiveControlWithStabilityandRobustnessGuarantees, Magni2004MPCOfContinuousTimeNonlinearSystemswithPiecewiseConstantControl}. 
Given the difficulty of manually engineering such constraints and the advent of data-driven methods,
learning from demonstration (LfD) has attracted increasing attention as a promising framework for recovering constraints from locally optimal and constraint-abiding expert trajectories \citep{ral20, ral22, quan2024learning, liu2023benchmarking}.
LfD-based constraint learning paradigms are often either based on inverse optimal control (IOC) \citep{toussaint, boyd, pontryagin, Rickenbach2024InverseOptimalControlasanErrorsinVariablesProblem, WulfmeierRWOP17, ral20, corl20, wafr} or inverse constrained reinforcement learning (ICRL) \citep{scobee2019maximum, mcpherson2021maximum, malik2021inverse, papadimitriou2022bayesian, liu2022distributed, liu2023meta, park2020inferring, jang2023inverse}. 
IOC-based methods recover constraints by enforcing consistency with the Karush–Kuhn–Tucker (KKT) conditions that characterize the local optimality of a given set of demonstrations \citep{Rickenbach2024InverseOptimalControlasanErrorsinVariablesProblem, menner, nguyen2026inversesafetyfilteringinferring, Zhang2026ConstraintLearningInMultiAgentDynamicGames, Chiu2026LearningConstraintsFromStochasticPartiallyObservedClosedLoopDemonstrations} under the dynamics models \citep{Rickenbach2024InverseOptimalControlasanErrorsinVariablesProblem, menner} and control methodology \citep{nguyen2026inversesafetyfilteringinferring, Zhang2026ConstraintLearningInMultiAgentDynamicGames, Chiu2026LearningConstraintsFromStochasticPartiallyObservedClosedLoopDemonstrations} from which the demonstrations were obtained.
However, such methods assume access to known, smooth system dynamics, which precludes their applicability to high-dimensional, black-box systems.
In contrast, ICRL-based constraint recovery \citep{scobee2019maximum, mcpherson2021maximum, malik2021inverse, papadimitriou2022bayesian, liu2022distributed, liu2023meta, park2020inferring, jang2023inverse} accommodates unknown dynamics and flexible neural constraint representations.
However, such methods usually require computation- \citep{mcpherson2021maximum, liu2023benchmarking} and data-heavy \citep{scobee2019maximum, liu2022distributed, quan2024learning} online exploration, may induce unsafe constraint violations during training \citep{malik2021inverse, liu2023benchmarking, quan2024learning}, and have few theoretical guarantees of constraint recovery accuracy \citep{malik2021inverse, park2020inferring, quan2024learning}.
While recent offline variants \citep{quan2024learning} avoid online exploration during training, they still rely on offline data generated by potentially constraint-violating policies to identify unsafe regions, which can be difficult or unsafe to collect in real-world systems.

In this work, we bridge CIOC and ICRL to enable constraint learning under unknown dynamics and constraint representations without requiring constraint-violating data. Our approach, \textit{Counterfactual KKT (CF-KKT)}, learns local dynamics from constraint-satisfying expert demonstrations and uses them to combine KKT-inspired optimality conditions with synthetic reward-improving counterfactual trajectories. These counterfactuals provide information about constraint-violating states without requiring constraint-violating trajectories to be executed in the environment. We further show that locally learned dynamics can support CIOC-based parameter recovery when the constraint parameterization is known.
Concretely, our contributions are:
\vspace{-5pt}
\begin{enumerate}[leftmargin=*, labelindent=0pt]
    \item For settings with closed-form \textit{a priori} known  constraint parameterizations, we present a two-stage framework which first learns system dynamics without additional demonstration data, then leverages the learned system dynamics to infer the unknown constraint via CIOC;
    %
    \item \looseness-1For settings with neural network (NN)-represented constraints, we introduce Counterfactual KKT (CF-KKT), which combines KKT-inspired optimality conditions with model-generated counterfactual trajectories to learn unknown constraints without observed constraint-violating data;
    \item We present a theoretical analysis of the sensitivity of the recovered constraints to system dynamics misspecification;
    %
    \item We evaluate CF-KKT in simulation across low- and high-dimensional systems with black-box dynamics, demonstrating improved constraint recovery and downstream safety over existing offline ICRL methods while structurally using 
    less data, all of which are safe.
\end{enumerate}

\vspace{-5pt}
\paragraph{Notation}
Given non-negative integers $n_1$ and $n_2$ with $n_1 < n_2$, we define $[n_2] := \{1, \cdots, n_2\}$ and $[n_1, n_2] := \{n_1, \cdots, n_2\}$. 
We use $C^k$ to denote $k$-times continuous differentiability.
Given a differentiable function $f: \R^n \ra \R^m$, we define the Jacobian matrix $\nabla_\boldx f \in \R^{m \times n}$ such that $[\nabla_\boldx f]_{ij} := \frac{\partial f_i}{\partial x_j}$ for each $i \in [m]$, $j \in [n]$.

\vspace{-6pt}
\section{Related Work}
\label{sec: Related Work}
\vspace{-6pt}


Our work relates most closely to constrained inverse optimal control (CIOC), which infers constraints from expert demonstrations using Karush–Kuhn–Tucker (KKT) optimality conditions \citep{toussaint, boyd, pontryagin, Rickenbach2024InverseOptimalControlasanErrorsinVariablesProblem, ral20, wafr, menner, nguyen2026inversesafetyfilteringinferring, Zhang2026ConstraintLearningInMultiAgentDynamicGames, corl, shah, Chiu2026LearningConstraintsFromStochasticPartiallyObservedClosedLoopDemonstrations}.
Existing CIOC-based constraint learning methods are typically highly data-efficient because they exploit the local optimality of demonstrations directly \citep{ral20}.
However, many such approaches assume \textit{known differentiable dynamics} and structured constraint parameterizations \citep{ral20, corl20}, which hampers their applicability to high-dimensional and neural network-represented systems.
While GP-based approaches relax parameterization assumptions \citep{ral22}, they scale poorly with dataset size and still require known dynamics.
In contrast, our work considers high-dimensional \textit{unknown dynamics} and \textit{unknown constraint representations} while preserving data efficiency, and enables further data-efficiency gains when the constraint parameterization is provided.

\looseness-1Partly in parallel with CIOC, ICRL-based methods \citep{liu2025comprehensive, liu2023benchmarking} were developed for inferring constraints from expert trajectories generated via constrained MDPs (CMDPs) \citep{malik2021inverse}. 
ICRL usually assumes access to a simulator rather than closed-form dynamics (in contrast to CIOC), and can be adapted to settings with unknown system dynamics or rewards \citep{liu2022distributed, liu2023meta, park2020inferring, jang2023inverse}. 
However, ICRL-based constraint recovery typically relies on iterative online environment exploration and constrained RL subroutines, leading to high computational cost, high sample complexity, and potentially unsafe exploration during training, which are unacceptable in time-sensitive and safety-critical settings.
Uncertainty-aware ICRL methods aim to reduce constraint violation risk \citep{xu2023uncertainty, subramanian2024confidence} but nonetheless still require online data. Offline ICRL \citep{liu2025comprehensive} avoids online exploration.
The primary prior offline method, ICSDICE \citep{quan2024learning}, constructs higher-reward constraint-violating distributions to identify unsafe regions (echoing earlier CIOC work by \cite{wafr}); however, doing so still requires a fixed dataset containing behaviors that are not guaranteed to satisfy the unknown constraint.

In contrast to the literature surveyed above, our work combines the safety and data efficiency advantage of CIOC with the representational flexibility of ICRL. Our approach leverages learned local dynamics to enable KKT-inspired inference under unknown dynamics, while synthetic reward-improving counterfactual data provide off-demonstration evidence of infeasibility without requiring observed constraint violations. In contrast, \cite{corl} uses hit-and-run sampling to generate counterfactuals but still relies heavily on known parametric constraints and explicit dynamics.

\vspace{-10pt}
\section{Preliminaries and Problem Formulation}
\label{sec: Preliminaries and Problem Formulation}
\vspace{-5pt}


We first review ICRL methods, with an emphasis on ICSDICE \citep{quan2024learning}, which serves as our primary offline ICRL baseline (Sec. \ref{subsec: Offline Inverse Constrained Reinforcement Learning (ICRL) via ICSDICE}). We then introduce the constrained inverse optimal control formulation and its associated KKT conditions (Sec. \ref{subsec: Constrained Inverse Optimal Control (CIOC)}), before formally stating the constraint-recovery problem considered in this work (Sec. \ref{subsec: Problem Statement}).





\vspace{-6pt}
\subsection{Inverse Constrained Reinforcement Learning (ICRL)}
\label{subsec: Offline Inverse Constrained Reinforcement Learning (ICRL) via ICSDICE}
\vspace{-6pt}
Consider 
a discrete-time, continuous-space constrained Markov decision process (CMDP) $(\mathcal{S},\mathcal{A},f,r,c,\bar\rho,\gamma)$ specified by a state space $\stateSpace \subseteq \R^n$, control (action) space $\controlSpace \subseteq \R^m$, system (transition) dynamics map $f: \stateSpace \times \controlSpace \ra \stateSpace$, reward map $r: \stateSpace \times \controlSpace \ra \R$, constraint map $c: \stateSpace \times \controlSpace \ra \R$, initial state distribution $\bar\rho$, and discount factor $\gamma \in [0, 1)$, 
with $f(\cdot, \cdot)$, $r(\cdot, \cdot)$, and $c(\cdot, \cdot)$ all in $C^1$.
Given a CMDP with \emph{known} constraint function $c(\cdot, \cdot)$, we consider the following horizon-$T$ \textit{forward constrained reinforcement learning (FCRL) problem}, which computes the optimal policy $\pi: \stateSpace \ra \controlSpace$ that maximizes, in expectation, the discounted cumulative reward $\sum_{t=0}^T \gamma^t r(x_t, u_t)$ subject to
the hard constraint $c(x_t, u_t) \leq 0$ over all times $t \in [0, T]$:
\begin{subequations}
\label{Eqn: Forward Constrained Reinforcement Learning (FCRL)}
\begin{align}
\label{Eqn: Forward Constrained Reinforcement Learning (FCRL), Objective}
  \pi^\star = \arg\max_{\pi: \stateSpace \ra \controlSpace} \quad
  &\mathbb{E}_{\bar\rho, \pi}
  \textstyle\left[\sum_{t=0}^T \gamma^t r\big(x_t, u_t \big)\right] \\
\label{Eqn: Forward Constrained Reinforcement Learning (FCRL), Constraint}
  \text{s.t.} \quad &\Prob_{\bar\rho, \pi}
  (c(x_t, u_t) \leq 0) = 1, \quad \forall t \in [0, T].
\end{align}
\end{subequations}
Above, the notations $\E_{\bar \rho, \pi}$ in \eqref{Eqn: Forward Constrained Reinforcement Learning (FCRL), Objective} and $\Prob_{\bar \rho, \pi}$ in \eqref{Eqn: Forward Constrained Reinforcement Learning (FCRL), Constraint} respectively denote expectations and probabilities taken over $x_0 \sim \bar \rho$ and $u_t \sim \pi$ for each $t \in [0, T]$, and, given  $\boldu := (u_0, \cdots, u_T)$ drawn component-wise from $\pi$, the states $\boldx := (x_0, \cdots, x_T)$ are defined to satisfy $x_{t+1} = f(x_t, u_t)$ for each $t \in [0, T-1]$.

In contrast,
the ICRL problem aims to recover \textit{a priori} unknown constraints $c(\cdot)$ from demonstrations of reward-maximizing and constraint-satisfying state-control trajectories.
Although current ICRL methods often assume that the demonstrations available for constraint learning are generated \textit{online} via continuous environment exploration (\say{interaction}), such exploratory actions are often costly or unsafe in safety-critical settings \citep{malik2021inverse, liu2023benchmarking, kim2023learning, adriana2024simplifyicrl}.
Offline ICRL methods include
ICSDICE \citep{quan2024learning}, which leverages two fixed demonstration datasets $\D_E$ and $\D_\pi$ for constraint inference, but does not allow environmental exploration during training.
Concretely, 
the \textit{expert dataset} $\D_E := \{(x_t^{\paren{d}}, u_t^{\paren{d}}, x_{t+1}^{\paren{d}}, r_t^{\paren{d}}): t \in [0, T-1], d \in [D_E]\}$ contains $D_E$ state-control trajectories generated using a constraint-satisfying optimal expert policy $\pi_E$, and a dataset $\D_{\pi} := \{(x_t^{\paren{d}}, u_t^{\paren{d}}, x_{t+1}^{\paren{d}}, r_t^{\paren{d}} ): t \in [0, T-1], d \in [D_{\pi}] \}$ contains $D_\pi$ state-control trajectories generated using possibly constraint-violating policies.

Although ICSDICE avoids online exploration during training, constructing $\D_\pi$ still requires behavior from policies that are not guaranteed to satisfy the unknown constraint.

 \vspace{-6pt}
\subsection{Constrained Inverse Optimal Control (CIOC)}
\label{subsec: Constrained Inverse Optimal Control (CIOC)}

Let $\stateSpace \subseteq \R^n$, $\controlSpace \subseteq \R^m$, $f: \stateSpace \times \controlSpace \ra \stateSpace$, $r: \stateSpace \times \controlSpace \ra \R$, and 
$c: \stateSpace \times \controlSpace \ra \R$
denote a state space, control (action) space, system dynamics map, reward map, and 
constraint map, respectively.
The horizon-$T$ \textit{constrained forward optimal control} (CFOC) problem, presented below, aims to compute optimal state-control trajectories $(\boldx, \boldu) := (x_0, \cdots, x_T, u_0, \cdots, u_T) \in \R^{(n+m)(T+1)}$ which minimize the cumulative cost $\sum_{t=0}^T - r(x_t, u_t)$ while satisfying the system dynamics $x_{t+1} = f(x_t, u_t)$ and constraints $c(x_t, u_t) \leq 0$ throughout $[0, T-1]$:
{
\setlength{\abovedisplayskip}{3pt}
\setlength{\belowdisplayskip}{3pt}
\begin{subequations}
\label{Eqn: Constrained Forward Optimal Control Problem}
\begin{align}
\label{Eqn: Constrained Forward Optimal Control Problem, Objective}
    \min_{\boldx, \boldu} \quad &\textstyle\sum_{t=0}^T -r(x_t, u_t) \\
\label{Eqn: Constrained Forward Optimal Control Problem, Constraints}
    \text{s.t.} \quad &x_{t+1} = f(x_t, u_t), \ \forall \ t \in [0, T-1], \quad c(x_t, u_t) \leq 0, \ \forall \ t \in [0, T].
\end{align}
\end{subequations}
}
$\hspace{-1mm}$Meanwhile, in the \textit{constrained inverse optimal control} (CIOC) problem, a constraint learner aims to recover \textit{a priori} unknown constraints $c(\cdot, \cdot)$
from a set of expert trajectories $\D_E := \{(x_t^\paren{d}, u_t^\paren{d}): t \in [0, T], d \in [D]\}$ which optimize \eqref{Eqn: Constrained Forward Optimal Control Problem}, assuming the reward map $r(\cdot, \cdot)$ is known. (The dynamics map $f(\cdot, \cdot)$ may either be \textit{a priori} known or unknown to the constraint learner.) The standard approach in CIOC for recovering the unknown constraint map $c(\cdot, \cdot)$ entails formulating the KKT equations corresponding to \eqref{Eqn: Constrained Forward Optimal Control Problem}, as presented below in \eqref{Eqn: KKT for Constrained Forward Optimal Control Problem}. 
Let $\boldlambda := (\lambda_0, \cdots, \lambda_T)$ and $\boldnu := (\nu_0, \cdots, \nu_{T-1})$ denote dual variables corresponding to the constraint and dynamics constraints in \eqref{Eqn: Constrained Forward Optimal Control Problem, Constraints}, respectively, across times $t$, and define the Lagrangian $\lagrangian: \stateSpace^{T+1} \times \controlSpace^{T+1} \times \R^{T+1} \times \R^{nT} \ra \R$ by:
\begin{align}
\label{Eqn: Lagrangian, Overall, CIOC}
    \lagrangian(\boldx, \boldu, \boldlambda, \boldnu) := \textstyle\sum_{t=0}^T -r(x_t, u_t) + \textstyle\sum_{t=0}^{T-1} \nu_t^\top (x_{t+1} - f(x_t, u_t)) + \textstyle\sum_{t=0}^T \lambda_t c(x_t, u_t).
\end{align}
The KKT equations for the CFOC problem \eqref{Eqn: Constrained Forward Optimal Control Problem} are thus:
{
\setlength{\abovedisplayskip}{3pt}
\setlength{\belowdisplayskip}{3pt}
\begin{subequations}
\label{Eqn: KKT for Constrained Forward Optimal Control Problem}
\begin{align}
\label{Eqn: KKT for Constrained Forward Optimal Control Problem, Complementarity}
    &c(x_t, u_t) \leq 0 \quad \perp \quad \lambda_t \geq 0, \qquad \forall t \in [0, T], \\
\label{Eqn: KKT for Constrained Forward Optimal Control Problem, Stationarity}
    &\nabla_{(\boldx, \boldu)} \lagrangian(\boldx, \boldu, \boldlambda, \boldnu) = 0.
\end{align}
\end{subequations}
}
$\hspace{-1mm}$Above, $\perp$ indicates complementarity of the left- and right-hand-side conditions. In \eqref{Eqn: KKT for Constrained Forward Optimal Control Problem, Complementarity}, $c(x_t,u_t)<0$ implies $\lambda_t=0$, while $\lambda_t>0$ implies $c(x_t,u_t)=0$; equivalently, $\lambda_t c(x_t,u_t)=0$.
In settings in which the constraint map $c(\cdot, \cdot)$ is \textit{a priori} known to be parameterized in a certain form,
the CIOC problem reduces to the recovery of parameter values for which the resulting constraint map $c(\cdot, \cdot)$ satisfies the KKT system \eqref{Eqn: KKT for Constrained Forward Optimal Control Problem}, when evaluated at expert state-control trajectories $(\boldx, \boldu)$ in $\D^E$ (see Sec. \ref{subsec: MIP Formulation for Parametric Constraints}).
Such parameter recovery problems can often be recast as mixed-integer programs (MIP) that, for certain constraint maps characterized by low-dimensional parameters, are efficient to solve exactly \citep{ral20}.

\paragraph{Comparing CIOC and ICRL Methods for Constraint Recovery}
CIOC exploits demonstration optimality directly, but typically requires known closed-form dynamics and structured constraint parameterizations. ICRL, in contrast, accommodates unknown dynamics and flexible constraint representations, but obtains off-expert coverage from additional behavior data beyond the expert demonstrations. This trade-off motivates the hybrid setting considered in this work.

\vspace{-5pt}
\subsection{Problem Statement}
\label{subsec: Problem Statement}
\vspace{-4pt}

We are given an expert dataset $\D_E$ of constraint-abiding optimal trajectories generated by solving a CFOC problem of the form \eqref{Eqn: Constrained Forward Optimal Control Problem}.
Moreover, suppose the reward map $r: \stateSpace \times \controlSpace \ra \R$ is known (we consider an unknown reward case in Appendix~\ref{subsec:exp-reward-model}) but the dynamics map $f: \stateSpace \times \controlSpace \ra \stateSpace$ and constraint map $c: \stateSpace \times \controlSpace \ra \R$ are unknown.
We aim to recover a neural network (NN) representation $\hat c(\cdot, \cdot; \theta): \stateSpace \times \controlSpace \ra \R$ for an \textit{a priori} unknown constraint map $c: \stateSpace \times \controlSpace \ra \R$, where $\theta \in \R^{M_\theta}$ describes NN weights to be trained using the dataset $\D_E$. 

\section{KKT-Informed Constraint Learning}
\label{sec: KKT-Informed Constraint Learning}
\vspace{-6pt}

We consider two constraint-recovery settings that share a learned local dynamics model. 
We first treat known parametric constraints with a CIOC-based mixed-integer program (MIP) in Sec. \ref{subsec: MIP Formulation for Parametric Constraints}, then remove the known parameterization and introduce \emph{Counterfactual KKT (CF-KKT)} for neural constraints in Sec. \ref{subsec: Neural KKT Constraint Learning}.

\vspace{-5pt}
\paragraph{Learned local dynamics.}
Using expert transitions $(x_t,u_t,x_{t+1})\in\D_E$, we fit
$\hat f(\cdot,\cdot;\psi):\stateSpace\times\controlSpace\ra\stateSpace$ by
$\loss_{\mathrm{dyn}}(\psi)=\E_{\D_E}[\|\hat f(x_t,u_t;\psi)-x_{t+1}\|_2^2]$, 
and use $\hat f(\cdot,\cdot;\psi)$ for both inference procedures below.
\vspace{-12pt}
\subsection{Learning Parametric Constraints via MIP}
\label{subsec: MIP Formulation for Parametric Constraints}

When the unknown constraint has a known closed-form parameterization $\hat c(\cdot,\cdot;\theta)$, we recover the unknown parameter $\theta$ by applying CIOC to the learned dynamics model. We define the full Lagrangian and its stationarity residual by:
{
\setlength{\abovedisplayskip}{3pt}
\setlength{\belowdisplayskip}{3pt}
\begin{align}
\label{Eqn: (Parameterized) Lagrangian, Overall, CIOC}
    \hspace{-8pt}\hat\lagrangian(\boldx, \boldu, \boldlambda, \boldnu; \theta, \psi) &:= \textstyle\sum_{t=0}^T -r(x_t, u_t) + \textstyle\sum_{t=0}^T \lambda_t \hat c(x_t, u_t; \theta) + \textstyle\sum_{t=0}^{T-1} \nu_t^\top (x_{t+1} - \hat f(x_t, u_t; \psi)),\hspace{-3pt} \\
    \hspace{-8pt}\hat\bolds(\boldx, \boldu, \boldlambda, \boldnu; \theta, \psi) &:= \nabla_{(\boldx, \boldu)} \hat\lagrangian(\boldx, \boldu, \boldlambda, \boldnu; \theta, \psi).\hspace{-3pt}
\end{align}
}
We seek constraint parameter values $\theta$ which are consistent with the KKT-optimality of the state-control trajectories $\{(\boldx^\paren{d}, \boldu^\paren{d}): d \in [D_E] \}$ in the expert demonstration dataset $\D^E$.
To this end, in \eqref{Eqn: Inverse KKT, Relaxed}, we pose a relaxed inverse KKT problem analogous to \citep[Prob. 2]{ral20}.
Concretely, we aim to recover parameters $\theta$ and dual variables $\boldlambda_{1:D} := (\boldlambda^\paren{d}: d \in [D_E])$ and $\boldnu_{1:D} := (\boldnu^\paren{d}: d \in [D_E])$ that minimize the sum, over all $d \in [D_E]$, of the magnitudes of stationarity residuals $\hat\bolds(\boldx^\paren{d}, \boldu^\paren{d}, \boldlambda^\paren{d}, \boldnu^\paren{d}; \theta, \psi)$, while adhering to the primal/dual feasibility and complementary slackness constraints in \eqref{Eqn: KKT for Constrained Forward Optimal Control Problem, Complementarity}:
{
\setlength{\abovedisplayskip}{3pt}
\setlength{\belowdisplayskip}{3pt}
\begin{subequations}
\label{Eqn: Inverse KKT, Relaxed}
\begin{align}
\label{Eqn: Inverse KKT, Relaxed, Objective}
    \min_{\theta, \boldlambda_{1:D_E}, \boldnu_{1:D_E}} 
    \quad &\textstyle\sum_{d=1}^{D_E} \big\Vert \hat \bolds(\boldx^\paren{d}, \boldu^\paren{d}, \boldlambda^\paren{d}, \boldnu^\paren{d}; \theta, \psi) \big\Vert_2^2 \\
\label{Eqn: Inverse KKT, Relaxed, Constraints}
    \text{s.t.} \hspace{1.2cm} &\hat c(x_t^\paren{d}, u_t^\paren{d}; \theta) \leq 0 \quad \perp \quad \lambda_t^\paren{d} \geq 0, \qquad \forall t \in [0, T], d \in [D_E].
\end{align}
\end{subequations}
}
$\hspace{-1mm}$In the CIOC-based constraint recovery literature \citep{ral20, Zhang2026ConstraintLearningInMultiAgentDynamicGames}, \eqref{Eqn: Inverse KKT, Relaxed} is often solved using off-the-shelf solvers such as \cite{gurobi}, with the constraints \eqref{Eqn: Inverse KKT, Relaxed, Constraints} encoded via big-M reformulations \citep{Bertsimas1998IntroductiontoLinearOptimization}. We provide an analysis of MIP sensitivity to dynamics misspecification in App. \ref{app:proofs}.

\vspace{-6pt}
\subsection{Counterfactual KKT (CF-KKT) Constraint Learning}
\label{subsec: Neural KKT Constraint Learning}
\vspace{-6pt}
When the constraint form is unknown, we represent it by a neural network $\hat c(\cdot,\cdot;\theta)$ and combine expert-side KKT supervision with synthetic counterfactual supervision.


\paragraph{Transition-wise auxiliary multipliers.}
The exact KKT stationarity condition in Sec.~\ref{subsec: Constrained Inverse Optimal Control (CIOC)} is defined over an entire trajectory and couples neighboring time steps through the dynamics multipliers $\boldnu=(\nu_0,\ldots,\nu_{T-1})$. 
Directly imposing this trajectory-level system is inconvenient for stochastic mini-batch training, where transition tuples $(x_t^\paren{d},u_t^\paren{d},x_{t+1}^\paren{d},r_t^\paren{d})$ are sampled across time indices and demonstrations.
We therefore construct a transition-wise KKT-informed objective using bounded auxiliary multipliers for the local inequality and dynamics terms, as defined below:
{
\setlength{\abovedisplayskip}{3pt}
\setlength{\belowdisplayskip}{3pt}
\begin{align}
\label{Eqn: Auxiliary Multipliers}
    \tilde \lambda_t(x_t,u_t;\theta)
    &=
    \lambda_{\max}
    \sigma\!\bigl(\hat c(x_t,u_t;\theta)\bigr),
    \qquad
    \tilde \nu_t
    =
    \nu_{\max}\tanh(\eta_t),
\end{align}
}
$\hspace{-1mm}$where $\lambda_{\max},\nu_{\max}>0$ are fixed bounds, $\eta_t\in\R^n$ is per-timestep, and $\tanh(\cdot)$ is applied elementwise. 
Above, \eqref{Eqn: Auxiliary Multipliers} 
ensures $\tilde \lambda_t \geq 0$, consistent with dual feasibility, and calibrates $\Vert \tilde \lambda_t \Vert_2$ such that $\tilde\lambda_t$ becomes small for states predicted well inside the feasible region and larger as $\hat c(x_t,u_t;\theta)$ approaches or exceeds zero.
The equality multiplier $\tilde\nu_t\in[-\nu_{\max},\nu_{\max}]^n$ can take either sign within a bounded range.
Bounding both multipliers improves numerical stability during training.

\paragraph{KKT-informed loss.}
For a single transition, define the local Lagrangian:
{
\setlength{\abovedisplayskip}{3pt}
\setlength{\belowdisplayskip}{3pt}
\begin{align}
\label{Eqn: Lagrangian, Per-step}
    \hat \lagrangian_t
    (x_t,u_t,x_{t+1},\lambda,\nu;\theta,\psi)
    &:=
    -r(x_t,u_t)
    +\lambda_{t}\hat c(x_t,u_t;\theta)
    +\nu^\top_t
    \bigl[
        \hat f(x_t,u_t;\psi)-x_{t+1}
    \bigr].
\end{align}
}
Given weights $w_p,w_c,w_a>0$, we define
{
\setlength{\abovedisplayskip}{3pt}
\setlength{\belowdisplayskip}{3pt}
\begin{align}
\label{Eqn: Loss, KKT}
\loss_\KKT(\theta)
=
\E_{\D_E}\big[
&
\big\|
\big.
\nabla_{x_t,u_t}
\hat\lagrangian_t
(x_t,u_t,x_{t+1},\lambda,\nu;\theta,\psi)
\big|_{\lambda=\tilde\lambda_t,
\nu=\tilde\nu_t}
\big\|_2^2
\nonumber\\
&+
w_p
\max\!\big\{
\hat c(x_t,u_t;\theta),0
\big\}^2
+
w_c
\big(
\tilde\lambda_t
\hat c(x_t,u_t;\theta)
\big)^2
+
w_a
\big|
\hat c(x_t,u_t;\theta)
\big|
\big].
\end{align}
}
$\hspace{-1mm}$The first three terms respectively penalize a transition-wise stationarity residual, primal constraint violation, and violation of complementary slackness. 
The final activity regularizer discourages the learned constraint from assigning large negative values to all expert transitions.

\begin{remark}
Transition tuples of the form
$(x_t^\paren{d}, u_t^\paren{d}, x_{t+1}^\paren{d})$
can be shuffled across time indices and trajectories to support
stochastic mini-batch training. The full trajectory KKT conditions,
however, include cross-time coupling through the dynamics
multipliers $\nu_t$. In our implementation, training with the full trajectory KKT residual produces Inf/NaN losses on \texttt{BlockedAnt} and \texttt{HalfCheetah}, while collapsing the cumulative reward on the other environments, which we attribute to the numerical instability of cross-time multiplier $\nu_t$ recursion.
Motivated by this observation and the computational convenience
of transition-wise training, we instead incorporate the per-time
Lagrangian \eqref{Eqn: Lagrangian, Per-step} into
$\loss_\KKT(\theta)$.
\end{remark}

\vspace{-10pt}

%

\paragraph{Synthetic Counterfactual Trajectories}
The KKT-informed loss in \eqref{Eqn: Loss, KKT} is evaluated only on expert
demonstrations, which are constraint-satisfying by construction. Consequently,
it provides limited direct supervision information about infeasible behaviors away from the
expert trajectories. To address this issue without collecting unsafe data,
CF-KKT uses the learned dynamics model to generate synthetic counterfactual
trajectories near the demonstrations.

We sample multiple windows of $H$ consecutive expert transitions from demonstrations. We denote the window from
trajectory $d$, starting at time index $i$ as
$\{(x_{i+h}^{\paren{d}},u_{i+h}^{\paren{d}})\}_{h=0}^{H-1}$.
For each window, we generate $K$ perturbed control sequences, indexed by
$k\in[K]$, by independently perturbing the corresponding expert control at
each rollout step, while holding $d$ fixed:
{
\setlength{\abovedisplayskip}{3pt}
\setlength{\belowdisplayskip}{3pt}
\begin{align}
\label{Eqn: Counterfactual Action Perturbation}
    u_h^\paren{k}
    =
    \mathrm{clip}
    \big(
    u_{i+h}^\paren{d}
    +
    \sigma\,\epsilon_h^\paren{d,k},
    -1,1
    \big),
    \qquad
    \epsilon_h^\paren{d,k}
    \sim \mathcal N(0,I_m).
\end{align}
}
\hspace{-1mm}
The perturbed and baseline rollouts start from the same demonstrated state,
$x_0^\paren{k}=\tilde x_0=x_i^\paren{d}$, and use the same learned dynamics.
The counterfactual rollout follows
$x_{h+1}^\paren{k}
=\hat f(x_h^\paren{k},u_h^\paren{k};\psi)$
for $h\in[0,H-1]$, while the baseline follows the corresponding expert action
sequence, with states rolled out through the learned dynamics:
$\tilde u_h=u_{i+h}^\paren{d}$ and
$\tilde x_{h+1}=\hat f(\tilde x_h,\tilde u_h;\psi)$.

A truncated $H$-step reward alone does not account for the cost-to-go of the
terminal state $x_H$, since a perturbation may obtain larger short-horizon reward
while steering the system towards a state with poor future cumulative reward. We therefore
augment the rollout reward with a value function estimate $\hat V(\cdot)$,
whose estimation from expert return-to-go is described in
App. \ref{sec: Experiment Details}:
{
\setlength{\abovedisplayskip}{3pt}
\setlength{\belowdisplayskip}{3pt}
\begin{align}
\label{Eqn: Counterfactual Return}
    \hat J_V^\paren{k}
    &=
    \textstyle\sum_{h=0}^{H-1}
    \gamma^h
    r
    \big(
    x_h^\paren{k},
    u_h^\paren{k}
    \big)
    +
    \gamma^H
    \hat V
    \big(
    x_H^\paren{k}
    \big),
    \\
    \hat J_{V,\mathrm{base}}
    &=
    \textstyle\sum_{h=0}^{H-1}
    \gamma^h
    r
    \left(
    \tilde x_h,
    \tilde u_h
    \right)
    +
    \gamma^H
    \hat V
    \left(
    \tilde x_H
    \right).
\end{align}
}
$\hspace{-1mm}$We then define the value-augmented counterfactual advantage as
$\hat A_V^\paren{k}:=\hat J_V^\paren{k}-\hat J_{V,\mathrm{base}}$.
The terminal-value term $\hat V(\cdot)$ accounts for consequences that extend
beyond the finite $H$-step rollout. Comparing the perturbed rollout against the
expert-action trajectory from the same initial state then isolates the effect of
the action perturbation, rather than differences due to the starting state.

We retain counterfactuals whose value-augmented return exceeds that of the
expert-action baseline, i.e., $w_k=\mathbf{1}\!\big[\hat A_V^\paren{k}>0\big]$.
Because both $\hat f$ and $\hat V$ are learned, $w_k$ is a practical approximation to an ideal selection rule. To state the approximate rule, let $V_{i+H}^\star(x)$ denote the optimal feasible discounted cost-to-go from state $x$ over the remainder of the expert horizon under the true dynamics,
with $V_{i+H}^\star(x)=-\infty$ if no feasible continuation exists. Let
$J_{V^\star}^\paren{k}$ and $J_{V^\star,\mathrm{base}}$ denote the two rewards
in \eqref{Eqn: Counterfactual Return} using true-dynamics rollouts and replacing
$\hat V$ by $V_{i+H}^\star$.

\begin{lemma}[Positive Counterfactual Advantage Implies Infeasibility]
\label{lem:counterfactual-infeasibility}
Let $(x^E_{0:T},u^E_{0:T})$ be a globally optimal feasible expert trajectory,
and consider an $H$-step window beginning at expert time $i$. Denote the
corresponding expert controls by
$u^E_{i:i+H-1}:=(u_i^E,\ldots,u_{i+H-1}^E)$.
For any alternative control sequence
$u_{0:H-1}:=(u_0,\ldots,u_{H-1})$ initialized at $x_0=x_i^E$, define
\begin{align}
J_{i,H}^V(u_{0:H-1})
=
\textstyle\sum_{h=0}^{H-1}\gamma^h r(x_h,u_h)
+
\gamma^H V_{i+H}^\star(x_H),
\end{align}
where $x_{h+1}=f(x_h,u_h)$ follows the true dynamics and
$V_{i+H}^\star(x_H)$ is the optimal constrained value achievable from $x_H$
over the remaining time steps $i+H,\ldots,T$. If no feasible trajectory exists
from $x_H$ over these remaining steps, we define
$V_{i+H}^\star(x_H)=-\infty$.
Then
{
\setlength{\abovedisplayskip}{3pt}
\setlength{\belowdisplayskip}{3pt}
\begin{align}
J_{i,H}^V(u_{0:H-1})
>
J_{i,H}^V(u^E_{i:i+H-1})
\quad\Longrightarrow\quad
\exists\,h\in\{0,\ldots,H-1\}:\;
c(x_h,u_h)>0.
\end{align}
}
\end{lemma}
\vspace{-12pt}

\begin{proof}
Let $J_{E,i}$ denote the expert return from time $i$ to $T$. Because the
expert trajectory is globally optimal, the expert controls from time $i$
onward are optimal when initialized at $x_i^E$. Hence,
\vspace{-6pt}
\[
J_{i,H}^V(u^E_{i:i+H-1})=J_{E,i}.
\]
\vspace{-20pt}

Now suppose that
$c(x_h,u_h)\leq0$ for every $h\in\{0,\ldots,H-1\}$.
If no feasible trajectory exists from $x_H$ over the remaining time steps,
then $J_{i,H}^V(u_{0:H-1})=-\infty$, so the strict inequality cannot hold.
Otherwise, following $u_{0:H-1}$ for $H$ steps and then taking the optimal
feasible controls from $x_H$ defines a feasible trajectory starting from
$x_i^E$. By global optimality of the expert,
\vspace{-6pt}
\[
J_{i,H}^V(u_{0:H-1})
\leq
J_{E,i}
=
J_{i,H}^V(u^E_{i:i+H-1}),
\]
which contradicts the assumed strict inequality.
\end{proof}
\vspace{-6pt}
The assumptions of Lemma~\ref{lem:counterfactual-infeasibility} are sufficient
conditions rather than requirements we expect to hold exactly in practice.
Accordingly, CF-KKT uses positive value-augmented advantage as a soft selection
criterion, whose empirical validity is supported by the
results in Appendix~\ref{app:cf-premise}.



\vspace{-6pt}
\paragraph{Counterfactual Constraint Loss}
Let
$\mathcal B^\paren{k}
=
\{(x_h^\paren{k},u_h^\paren{k})\}_{h=0}^{H-1}$
denote the $k$-th counterfactual set. For the selected sets, we train the average learned constraint value along the rollout to exceed a positive margin $m$:
{
\setlength{\abovedisplayskip}{3pt}
\setlength{\belowdisplayskip}{3pt}
\begin{align}
\label{Eqn: Loss, Counterfactual}
    \loss_{\mathrm{CF}}(\theta)
    =
    \frac{
    \sum_{k=1}^{K}
    w_k
    \cdot
    \max
    \big\{
    m
    -
\frac{1}{H}\sum_{h=0}^{H-1}
    \hat c
    \big(
    x_h^\paren{k},
    u_h^\paren{k};
    \theta
    \big),
    0
    \big\}^2
    }{
    \sum_{k=1}^{K} w_k
    }.
\end{align}
}
$\hspace{-1mm}$The counterfactual trajectories are treated as fixed synthetic samples during the constraint update, so gradients of \eqref{Eqn: Loss, Counterfactual} propagate only through $\hat c(\cdot,\cdot;\theta)$. By Lemma~\ref{lem:counterfactual-infeasibility}, a counterfactual trajectory which generates a higher reward than a constrained optimal trajectory must violate constraint feasibility at one or more timesteps. 
We nevertheless use the mean constraint value along the rollout to provide denser information. 
This is motivated by the empirical observation that the rollout may violate feasibility at multiple timesteps. 
In particular, App. \ref{app:cf-premise} shows high state-level precision among selected counterfactuals, indicating that a substantial fraction of states along a positive-advantage rollout are indeed infeasible. 
Averaging therefore provides a denser and smoother loss gradient than a max-based or single-step objective that uses fewer steps to train.
Finally, CF-KKT combines the original KKT-informed loss with the synthetic counterfactual loss:
{
\setlength{\abovedisplayskip}{3pt}
\setlength{\belowdisplayskip}{3pt}
\begin{align}
\label{Eqn: Loss, CF-KKT}
    \loss_{\mathrm{CF\text{-}KKT}}(\theta)
    =
    \alpha_{\KKT}
    \loss_\KKT(\theta)
    +
    \alpha_{\mathrm{CF}}
    \loss_{\mathrm{CF}}(\theta).
\end{align}
}


\vspace{-10pt}
\section{Experiments}
\label{sec: Experiments}
\vspace{-8pt}

We evaluate three increasingly unstructured settings: known parametric constraints under unknown dynamics, neural constraint recovery in continuous mazes, and high-dimensional MuJoCo tasks with unknown dynamics and constraint geometry. Our methods infer constraints from expert demonstrations, using learned dynamics to provide synthetic off-demonstration supervision when needed.
\vspace{-20pt}
\paragraph{Baselines.}
We compare against \textbf{ICSDICE} \citep{quan2024learning},
\textbf{SMODICE} \citep{ma2022versatile}, and
\textbf{IQ-Learn} \citep{garg2021iq}.
ICSDICE is our primary offline ICRL baseline and explicitly learns a constraint
from expert and suboptimal offline data. SMODICE is an offline imitation method
based on distribution matching and does not natively recover a constraint mask $\mathbf{1}[c(x)\leq0]$; for constraint-recovery and downstream-planning comparisons, we modify
SMODICE to convert its learned state score into a feasibility mask.
IQ-Learn likewise does not recover an explicit constraint; we therefore keep
its standard imitation-learning formulation and evaluate the learned policy
directly. In the known-parametric setting, we additionally compare against a
parametric variant of ICSDICE using the same constraint family as the MIP method.
\vspace{-6pt}
\subsection{Known Parametric Constraints under Unknown Dynamics}
\label{subsec:exp-mip}
\vspace{-5pt}

We first consider the structured setting in which the constraint family is known up to a scalar parameter $\theta$ but the dynamics are unknown. We evaluate \texttt{HalfCheetah}, \texttt{BlockedAnt}, \texttt{Swimmer}, and \texttt{Quadcopter}, using one expert trajectory to fit $\hat f$ and recover $\theta$ via the MIP in Sec.~\ref{subsec: MIP Formulation for Parametric Constraints}. Environment, constraint, and demonstration details are given in App. \ref{app:parametric-envs}.
\vspace{-10pt}
\paragraph{Comparison against Parametric ICSDICE.}
For \texttt{HalfCheetah} and \texttt{BlockedAnt}, we also train ICSDICE with the same scalar constraint family $c_\theta(x)=\theta-x_{\mathrm{pos}}$, with $\mathrm{pos}$ being position on the x-axis. MIP method uses the expert trajectory, whereas ICSDICE additionally uses suboptimal trajectories from~\citep{liu2023benchmarking}.
 
Table~\ref{tab:mip_results} shows that the MIP closely recovers the true value of $\theta$ on \texttt{Swimmer}, \texttt{Quadcopter}, and \texttt{HalfCheetah}.
The larger \texttt{BlockedAnt} error ($-2.05$ versus $-3.0$) is consistent with
larger KKT residuals in its demonstrations. Parametric ICSDICE yields $-1.23$
on \texttt{HalfCheetah} and $-3.99$ on \texttt{BlockedAnt}, showing that the
same parametric family can produce substantially different estimates under an
occupancy-based objective.
\vspace{-6pt}

\begin{wraptable}{r}{0.50\columnwidth}
\vspace{-10pt}
\centering
\scriptsize
\setlength{\tabcolsep}{2.5pt}
\renewcommand{\arraystretch}{0.9}
\resizebox{\linewidth}{!}{%
\begin{tabular}{lcccc}
\toprule
Env. & $\theta^*$ & MIP $\hat\theta$ & ICSDICE $\hat\theta$ & MIP error \\
\midrule
Swimmer     & $0.04$ & $0.0484$ & --      & $+0.0084$ \\
Quadcopter  & $64$   & $64$     & --      & $0$ \\
HalfCheetah & $-3.0$ & $-2.98$  & $-1.23$ & $+0.02$ \\
BlockedAnt  & $-3.0$ & $-2.05$  & $-3.99$ & $+0.95$ \\
\bottomrule
\end{tabular}%
}
\vspace{-10pt}
\caption{Parametric recovery performance}
\label{tab:mip_results}
\vspace{-15pt}
\end{wraptable}

\subsection{Unknown Constraint Geometry in Continuous Mazes}
\label{subsec:exp-maze}
\vspace{-6pt}
We next remove the assumption of known constraint parameterization while keeping the state
dimension small enough to limit constraint complexity. These
continuous maze environments let us test whether CF-KKT can recover qualitatively
different constraint geometries, rather than only a scalar parameter within a
known family.
\vspace{-6pt}
\paragraph{Environment and Demonstrations.}
We use four 2D navigation tasks (\texttt{Obstacle}, \texttt{Wide},
\texttt{Disk}, \texttt{Zigzag}) with double-integrator dynamics and 
unknown obstacle geometries. CF-KKT receives feasible expert demonstrations;
baselines additionally receive a suboptimal dataset that may contain
violations. For more dynamics, reward, geometry, and dataset-generation details, see App. \ref{app:maze-envs}.
\vspace{-12pt}
\paragraph{Constraint IoU.}
Let $\mathcal U^\star$ be the ground-truth unsafe region and
$\hat{\mathcal U}=\{x:\hat c(x;\theta)>\tau\}$ the predicted unsafe set. We
report $\mathrm{IoU}=|\hat{\mathcal U}\cap\mathcal U^\star|/
|\hat{\mathcal U}\cup\mathcal U^\star|$ on a uniform $161\times161$ position
grid with zero velocity. To account for method-dependent score scales, we use
the common percentile threshold
$\tau=Q_{0.85}(\{\hat c(x;\theta):x\in\mathrm{grid}\})$. Table~\ref{tab:iou} and Fig.~\ref{fig:cost_map} show that CF-KKT recovers the
constraint geometry more accurately on every environment. Its mean IoU is
$0.707$, compared with $0.425$ for ICSDICE and $0.379$ for SMODICE. The
largest gains over ICSDICE occur on \texttt{Disk}
($0.815$ versus $0.428$) and \texttt{Obstacle}
($0.590$ versus $0.276$).

\begin{wraptable}{r}{0.53\columnwidth}
\vspace{-10pt}
\centering
\scriptsize
\setlength{\tabcolsep}{2.2pt}
\renewcommand{\arraystretch}{0.9}
\resizebox{\linewidth}{!}{%
\begin{tabular}{lcccc}
\toprule
Env. & SMODICE & ICSDICE & CF-KKT & $\Delta$ \\
\midrule
Obstacle & $0.275\!\pm\!0.405$ & $0.276\!\pm\!0.193$ & $\mathbf{0.590\!\pm\!0.153}$ & $+0.314$ \\
Wide     & $0.547\!\pm\!0.113$ & $0.600\!\pm\!0.140$ & $\mathbf{0.724\!\pm\!0.083}$ & $+0.124$ \\
Disk     & $0.363\!\pm\!0.133$ & $0.428\!\pm\!0.218$ & $\mathbf{0.815\!\pm\!0.057}$ & $+0.387$ \\
Zigzag   & $0.330\!\pm\!0.054$ & $0.394\!\pm\!0.172$ & $\mathbf{0.698\!\pm\!0.079}$ & $+0.304$ \\
\midrule
Mean     & $0.379$ & $0.425$ & $\mathbf{0.707}$ & $+0.282$ \\
\bottomrule
\end{tabular}%
}
\vspace{-10pt}
\caption{Maze constraint IoU (mean $\pm$ std over seeds); $\Delta$ is relative to ICSDICE.}
\label{tab:iou}
\vspace{-15pt}
\end{wraptable}

\begin{figure}[t]
    \centering
    \includegraphics[width=0.9\linewidth]{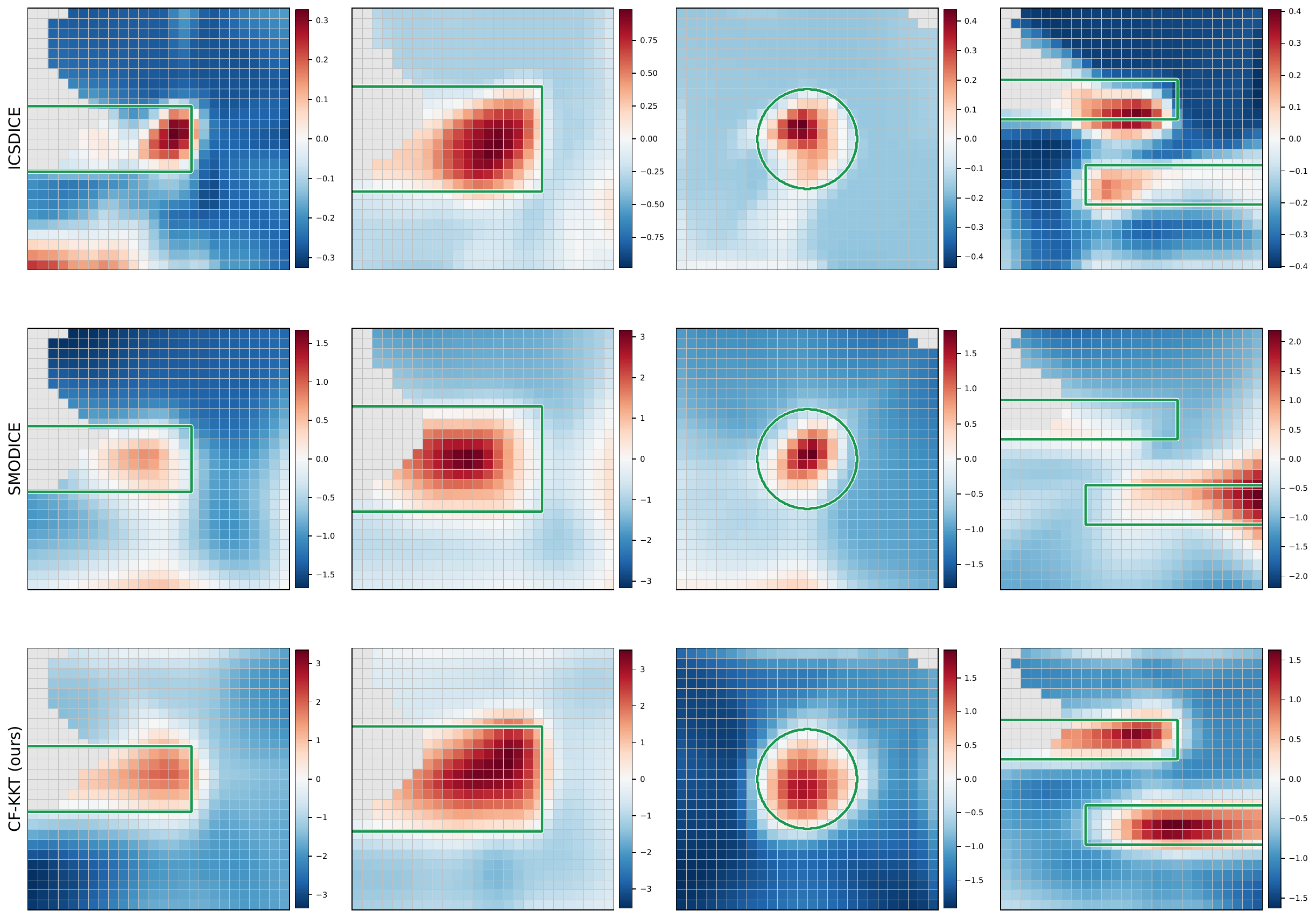}
    \vspace{-6pt}
    \caption{Ground-truth and learned constraint maps for 4 maze environments across ICSDICE (top), SMODICE (middle), and CF-KKT (bottom). 
    Red, blue, and green denote predicted unsafe states, predicted safe states, and the ground-truth constraint boundary.
    }
    \vspace{-15pt}
    \label{fig:cost_map}
\end{figure}
\vspace{-15pt}
\paragraph{Downstream Planning.}
We convert each learned constraint to a $61\times61$ occupancy grid, plan to the
goal with Dijkstra over predicted feasible cells, and track the waypoints with
the expert PD controller. Across 20 feasible starts per seed, we report goal
reaching, return (cumulative reward), and ground-truth violation rate
$\mathrm{Violation}=T^{-1}\sum_{t=0}^{T-1}\mathbf{1}[c(x_t,u_t)>0]$.
IQ-Learn does not recover a mask, so we evaluate its imitation policy directly. Representative planning examples are shown in App. \ref{app:maze-planning-vis}.

Table~\ref{tab:planning} shows the return-violation table. IQ-Learn
achieves a violation rate of $0.001$ on \texttt{Wide} but reaches the goal from
only $5/20$ starts. The oracle planner itself has nonzero violation on the
rectangular environments because PD tracking can clip corners, which provides
a planner-dependent floor. Relative to ICSDICE, CF-KKT lowers violation on all
four environments while retaining comparable goal-reaching performance:
$0.174\!\to\!0.042$ on \texttt{Obstacle},
$0.200\!\to\!0.085$ on \texttt{Wide},
$0.140\!\to\!0.080$ on \texttt{Disk}, and
$0.209\!\to\!0.093$ on \texttt{Zigzag}.

\begin{table}[t]
\centering
\scriptsize
\setlength{\tabcolsep}{2.5pt}
\renewcommand{\arraystretch}{0.88}
\resizebox{0.8\linewidth}{!}{%
\begin{tabular}{llccccc}
\toprule
Env. & Metric & Oracle & IQ-Learn & SMODICE & ICSDICE & CF-KKT \\
\midrule
\multirow{3}{*}{Obstacle}
& Goal & $20.0/20$ & $10.0/20$ & $18.3/20$ & $18.7/20$ & $20.0/20$ \\
& Viol. & $0.018$ & $0.124\pm0.203$ & $0.175\pm0.053$ & $0.174\pm0.038$ & $\mathbf{0.042\pm0.068}$ \\
& Return & $-58.2$ & $-89.7\pm5.8$ & $-47.1\pm7.4$ & $-45.8\pm1.6$ & $-55.6\pm6.8$ \\
\midrule
\multirow{3}{*}{Wide}
& Goal & $19.0/20$ & $5.0/20$ & $17.7/20$ & $18.3/20$ & $17.7/20$ \\
& Viol. & $0.032$ & $0.001\pm0.002$ & $0.226\pm0.111$ & $0.200\pm0.028$ & $\mathbf{0.085\pm0.121}$ \\
& Return & $-70.2$ & $-106.5\pm11.3$ & $-58.2\pm7.3$ & $-63.4\pm7.0$ & $-68.5\pm11.4$ \\
\midrule
\multirow{3}{*}{Disk}
& Goal & $20.0/20$ & $14.3/20$ & $19.3/20$ & $19.3/20$ & $20.0/20$ \\
& Viol. & $0.000$ & $0.041\pm0.039$ & $0.181\pm0.016$ & $0.140\pm0.060$ & $\mathbf{0.080\pm0.019}$ \\
& Return & $-48.5$ & $-84.0\pm2.4$ & $-46.4\pm0.6$ & $-47.1\pm0.8$ & $-48.3\pm0.5$ \\
\midrule
\multirow{3}{*}{Zigzag}
& Goal & $20.0/20$ & $7.0/20$ & $18.3/20$ & $17.3/20$ & $18.7/20$ \\
& Viol. & $0.060$ & $0.086\pm0.084$ & $0.230\pm0.033$ & $0.209\pm0.052$ & $\mathbf{0.093\pm0.044}$ \\
& Return & $-68.3$ & $-92.3\pm5.8$ & $-45.6\pm5.4$ & $-56.4\pm7.2$ & $-62.4\pm1.6$ \\
\bottomrule
\end{tabular}}
\vspace{-6pt}
\caption{Maze downstream planning over 20 starts.``Viol.'' is the true violating-timestep fraction.}
\label{tab:planning}
\vspace{-6pt}
\end{table}

\vspace{-12pt}
\subsection{High-Dimensional Neural Constraint Learning}
\label{subsec:exp-mujoco}

\vspace{-5pt}
Finally, we evaluate our method over the fully unstructured high-dimensional setting, where the transition dynamics are unknown and the constraint is represented by a neural
network rather than a known parametric family.
\vspace{-10pt}
\paragraph{Datasets.}
We use the five continuous-control environments and offline datasets from 
ICSDICE  \citep{quan2024learning}:
\texttt{Ant\_ls}, \texttt{BlockedAnt}, \texttt{BlockedHopper}, \texttt{HalfCheetahWithObstacle}, and \texttt{Walker\_ls}. 
Each environment has a hidden ground-truth constraint used only for evaluation. 
The benchmark contains $50{,}000$ expert transitions in $\D_E$ and $200{,}000$ suboptimal transitions in $\D_\pi$.
Under the data protocol above, the CF-KKT constraint learner receives $\D_E$, while SMODICE, IQ-Learn, and ICSDICE use 
more data, i.e.,
$\D_E \cup \D_\pi$. 
\vspace{-3mm}
\paragraph{Downstream Offline Policy Learning.}
Following the offline policy extraction
procedure of~\cite{quan2024learning}, we train the downstream value function and
advantage-weighted policy, while filtering
transitions according to the learned constraint. For the state-only constraints in these environments, we use the feasibility mask $M(x_{t+1})=\mathbf{1}[\hat c(x_{t+1};\theta)\leq0]$.
Transitions predicted to be unsafe are excluded from the value and
advantage-weighted policy updates. 

\vspace{-3mm}


\paragraph{Evaluation Metrics.}
For downstream control, we report \textbf{constraint violations}, defined as
the number of violating timesteps in a 1000-step evaluation episode, and
\textbf{return}, the cumulative task reward. Results are averaged
over $10$ random seeds.
\vspace{-3mm}
\paragraph{Downstream Safety and Return.}
Table~\ref{tab:mujoco_main} shows that CF-KKT reduces mean violations relative
to ICSDICE on all five tasks, by $34.6\%$--$96.8\%$. On \texttt{BlockedAnt}
and \texttt{BlockedHopper}, it also improves mean return. The remaining tasks
exhibit a safety--return trade-off; on \texttt{Walker\_ls}, SMODICE obtains
fewer violations but substantially lower return. Thus, the learned CF-KKT mask
consistently improves safety relative to ICSDICE while preserving useful task
performance.
\vspace{-10pt}
\begin{table}[t]
\centering
\scriptsize
\setlength{\tabcolsep}{1.8pt}
\renewcommand{\arraystretch}{0.9}
\resizebox{\linewidth}{!}{%
\begin{tabular}{lcccccccc}
\toprule
& \multicolumn{2}{c}{ICSDICE} & \multicolumn{2}{c}{IQ-Learn}
& \multicolumn{2}{c}{SMODICE} & \multicolumn{2}{c}{CF-KKT} \\
\cmidrule(lr){2-3}\cmidrule(lr){4-5}\cmidrule(lr){6-7}\cmidrule(lr){8-9}
Env. & Ret. & Viol. & Ret. & Viol. & Ret. & Viol. & Ret. & Viol. \\
\midrule
Ant\_ls
& $1257\pm123$ & $87.5\pm117.0$ & $65\pm155^{\dagger}$ & $6.3\pm1.8$
& $970\pm186$ & $86.8\pm113.5$ & $1130\pm91$ & $\mathbf{2.8\pm0.9}$ \\
BlockedAnt
& $2265\pm400$ & $82.3\pm56.1$ & $258\pm316^{\dagger}$ & $145.3\pm135.2$
& $1323\pm537^{\dagger}$ & $124.2\pm84.7$ & $2449\pm237$ & $\mathbf{15.3\pm4.2}$ \\
BlockedHopper
& $420\pm112$ & $110.6\pm26.6$ & $269\pm274^{\dagger}$ & $134.1\pm260.6$
& $325\pm275$ & $82.9\pm63.4$ & $454\pm252$ & $\mathbf{17.6\pm21.2}$ \\
HalfCheetah
& $2657\pm318$ & $83.0\pm58.9$ & $-270\pm447^{\dagger}$ & $0.0\pm0.0$
& $1915\pm977$ & $313.1\pm349.0$ & $2338\pm788$ & $\mathbf{54.3\pm48.9}$ \\
Walker\_ls
& $682\pm392$ & $170.1\pm106.9$ & $361\pm233$ & $39.0\pm16.8$
& $140\pm224$ & $\mathbf{37.6\pm56.5}$ & $410\pm283$ & $74.3\pm56.3$ \\
\bottomrule
\end{tabular}}
\vspace{-6pt}
\caption{MuJoCo downstream return and violations ($N=10$). Bold marks the lowest mean violation among explicit-mask methods; $^{\dagger}$ denotes severe return collapse.}
\label{tab:mujoco_main}
\vspace{-15pt}

\end{table}

\section{Conclusion}
\label{sec: Conclusion}
\vspace{-6pt}

We present a framework for learning unknown constraints from expert demonstrations under unknown dynamics. 
For known parametric constraints, we combine learned local dynamics with CIOC-based recovery and analyze sensitivity to model error. For neural constraints, CF-KKT combines KKT-inspired expert supervision with synthetic reward-improving counterfactuals to provide off-demonstration constraint information without violations.
Experiments on robotic control tasks show improved constraint recovery and downstream safety over offline ICRL baselines. 
Limitations include environment-dependent loss weights, the assumption of a known differentiable reward, omission of cross-time coupling in the transition-wise KKT objective, and dependence of counterfactual quality on learned dynamics away from the demonstrations. 
Automating loss weighting and incorporating reward and dynamics uncertainty are promising directions for future work.


\section*{Generative AI Use Disclosure}

Generative AI tools were used in the preparation of this work to assist with
refining mathematical statements and annotations, providing feedback on human-designed experimental
methodology and ablation, and
improving the organization, clarity, and presentation of the manuscript.
They were also used for language editing, restructuring sections, and
formatting equations and tables.

All AI-assisted mathematical arguments were reviewed and revised by the
authors, and all reported experimental results were produced and verified by
the authors. Generative AI was not treated as a source of scientific evidence
or authority; technical claims, derivations, and conclusions were checked
against the underlying analysis and experimental results. The authors take full
responsibility for the final content of the paper, including all text, claims,
and artifacts produced with the assistance of generative AI.

\bibliography{references}
\bibliographystyle{iclr2027_conference}

\appendix
\clearpage

\section{Benchmark Environments and Datasets}
\label{app:envs}
This appendix collects environment and dataset details omitted from the main experimental section.

\subsection{Known-Parametric Experiments}
\label{app:parametric-envs}
For \texttt{HalfCheetah} and \texttt{BlockedAnt}, the unknown constraint is the position lower bound $c(x;\theta)=\theta-x_{\mathrm{pos}}\leq0$, with true threshold $\theta^*=-3.0$. 
Their trajectory data comes from the ICRL benchmark of~\cite{liu2023benchmarking}. For \texttt{Swimmer} and \texttt{Quadcopter}, the constraint encodes spherical obstacle avoidance, $c(x;\theta)=\theta-\|p(x)-p_{\mathrm{obs}}\|_2^2\leq0$, where $p(x)$ is the position of the designated body (the swimmer nose or quadcopter center of mass) and $p_{\mathrm{obs}}$ is the known obstacle center. 
The true squared radii are $\theta^*=0.04$ and $64$, respectively. Their expert trajectories are generated offline with MuJoCo-MPC \citep{howell2022}. 
The parametric experiment uses one expert trajectory per environment to fit the local dynamics and recover $\theta$.

\begin{figure}[t]
    \centering
    \includegraphics[width=0.99\textwidth]{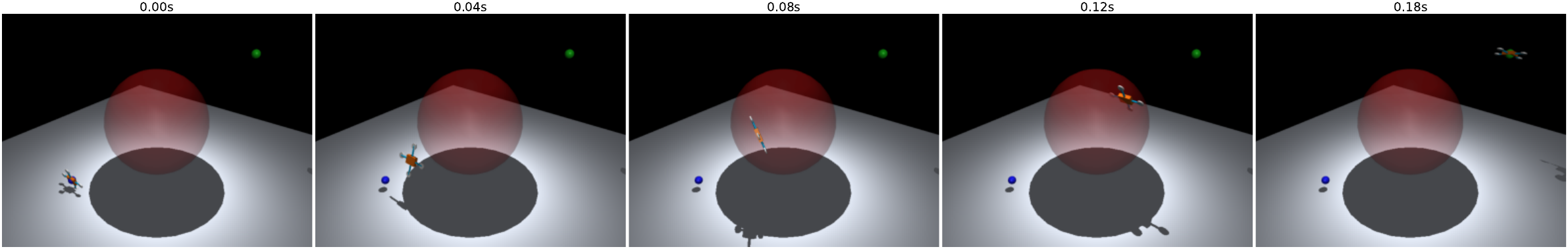}
    \vspace{2pt}
    \includegraphics[width=0.99\textwidth]{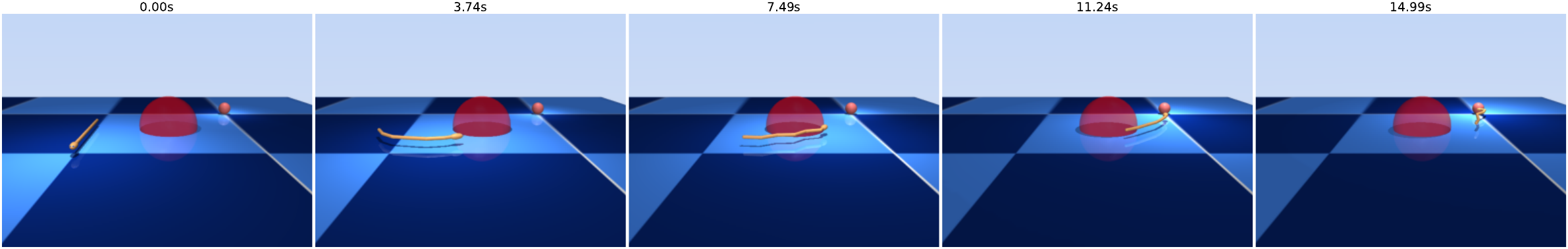}
    \caption{\texttt{Quadcopter} (top) and \texttt{Swimmer} (bottom) goal-reaching trajectories with the learned collision-avoidance constraint; red denotes the unsafe set.}
    \label{fig:timelapse}
\end{figure}

\subsection{Continuous-Maze Experiments}
\label{app:maze-envs}
We use a double-integrator system on
$\stateSpace=[-1,1]^2\times\R^2$, with $x_t=(p_t,v_t)\in\R^4$ and
$u_t\in[-1,1]^2$:
\begin{align}
    v_{t+1}&=(1-\kappa)v_t+\Delta t\,\alpha u_t,
    \qquad \|v_{t+1}\|_2\leq v_{\max},\\
    p_{t+1}&=\Pi_{\mathcal X}(p_t+\Delta t\,v_{t+1}),
\end{align}
with $\Delta t=0.1$, $\alpha=1$, $\kappa=0.15$, and $v_{\max}=2$. The reward is
$r(x_t,u_t)=-\|p_t-g\|_2-0.02\|u_t\|_2^2$ with $\gamma=0.99$. We consider four
unknown state-only unsafe regions---\texttt{Obstacle}, \texttt{Wide},
\texttt{Disk}, and \texttt{Zigzag}---covering wall-attached, compact, and
non-convex geometries. Although the ground-truth constraint depends only on
position, the learner is not given this structural knowledge.

Each environment and seed contains approximately $30{,}000$ feasible expert
transitions. Expert paths are generated by Dijkstra on a $61\times61$ occupancy
grid and tracked by a PD controller. IQ-Learn, SMODICE, and ICSDICE additionally
receive a $7000$-transition suboptimal dataset that may contain constraint
violations.

\subsection{MuJoCo Neural-Constraint Experiments}
\label{app:mujoco-envs}
\paragraph{Task specifications.}We evaluate CF-KKT on five continuous-control tasks built from MuJoCo
locomotion environments in Gymnasium (\texttt{-v4}), following the benchmark
used by ICSDICE \citep{quan2024learning}. Each task augments the nominal
locomotion objective with a state constraint that is not explicitly penalized
by the task reward.

\begin{table}[h]
\centering
\small
\setlength{\tabcolsep}{4pt}
\begin{tabular}{llccl}
\toprule
\textbf{Task} & \textbf{Base environment} & $\dim(\stateSpace)$ & $\dim(\controlSpace)$
 & \textbf{Constraint: infeasible when} \\
\midrule
Ant\_ls       & Ant-v4         & 27 & 8 & $v_x > 0.5$ \\
BlockedAnt    & Ant-v4         & 27 & 8 & $\sum_{j \in \mathcal{J}} |\omega_j| > 1$ \\
BlockedHopper & Hopper-v4      & 11 & 3 & $|\phi_{\mathrm{foot}}| > 0.3$ \\
HalfCheetah   & HalfCheetah-v4 & 18 & 6 & $x_{\mathrm{pos}} < -3$ \\
Walker\_ls    & Walker2d-v4    & 17 & 6 & $v_x > 1$ \\
\bottomrule
\end{tabular}
\caption{MuJoCo benchmark specifications. Here $v_x$ denotes forward velocity,
$\omega_j$ the angular velocity of joint $j$ with $\mathcal J$ denoting the
four constrained joints, $\phi_{\mathrm{foot}}$ the foot-joint angle, and
$x_{\mathrm{pos}}$ the absolute horizontal torso position.}
\label{tab:env-specs}
\end{table}

\begin{table}[h]
\centering
\small
\begin{tabular}{lll}
\toprule
\textbf{Task} & \textbf{Reward} & \textbf{Constrained quantity} \\
\midrule
Ant\_ls       & $v_x + 1 - 0.5\|u\|^2$    & forward velocity (rewarded) \\
BlockedAnt    & $v_x + 1 - 0.5\|u\|^2$    & joint angular velocities \\
BlockedHopper & $v_x + 1 - 10^{-3}\|u\|^2$ & foot-joint angle \\
HalfCheetah   & $|v_x| - 0.1\|u\|^2$      & horizontal position \\
Walker\_ls    & $v_x + 1 - 10^{-3}\|u\|^2$ & forward velocity (rewarded) \\
\bottomrule
\end{tabular}
\caption{Task rewards and constrained physical quantities. The constant $+1$
is the survival bonus in the corresponding locomotion task.
\texttt{HalfCheetah} rewards $|v_x|$, so motion in either horizontal direction
contributes positively to the task reward.}
\label{tab:env-rewards}
\end{table}

\paragraph{Relationship between reward and constraint.}

The five tasks differ in how directly reward-improving behavior drives the
constrained physical quantity. This distinction helps explain the
counterfactual diagnostics in Appendix~\ref{app:cf-premise} and the supervision
ablation in Appendix~\ref{app:ablation}.

\textit{Direct coupling.}
In \texttt{Ant\_ls} and \texttt{Walker\_ls}, the constraint directly bounds
forward velocity, which also appears positively in the task reward. Reward
improvement therefore tends to move the system toward the constraint boundary,
making positive-value-advantage counterfactuals informative for constraint
inference.

\textit{Indirect coupling.}
\texttt{BlockedAnt} constrains joint angular velocities while rewarding forward
body velocity. These quantities are correlated through the locomotion dynamics,
but the constrained quantity does not appear directly in the reward.

\textit{Delayed angular coupling.}
\texttt{BlockedHopper} constrains a joint angle, which evolves over multiple
dynamics steps rather than changing instantaneously with the control. As a
result, violations can emerge progressively rather than at the first
counterfactual step.

\textit{Delayed positional coupling.}
\texttt{HalfCheetah} constrains absolute horizontal position, whereas action
perturbations affect position indirectly through the evolution of velocity.
Multi-step rollouts are therefore important for allowing perturbations to
propagate to the constrained quantity. Appendix~\ref{app:cf-premise}
quantifies these coupling effects using counterfactual violation diagnostics.

\paragraph{Datasets.}

Each task provides approximately $5\times10^4$ expert transitions in $\D_E$,
organized into $50$ trajectories, together with $2\times10^5$ transitions in
the benchmark suboptimal dataset $\D_\pi$.
\texttt{Hopper} and \texttt{Walker2d} can terminate early when the agent falls,
so their expert trajectories need not have equal length
(e.g., \texttt{BlockedHopper} contains $49{,}945$ transitions across $50$
episodes); the remaining tasks use fixed $1000$-step episodes. Ground-truth
constraint labels are reserved for evaluation and oracle diagnostics.

\clearpage
\section{Experiment Details}
\label{sec: Experiment Details}

\paragraph{Experiment Setup.}
All networks are two-layer MLPs with 256 hidden units and ReLU activations,
implemented in JAX/Flax. The constraint network $c_\theta$, dynamics model
$f_\psi$, continuation-value model $\hat V$, and downstream policy $\pi_\phi$
use this architecture. All networks are trained with Adam at a learning rate
of $3\times10^{-4}$. We use a mini-batch size of $256$ and discount factor
$\gamma=0.99$. For the KKT terms, the default weights are
$w_p=10.0$, $w_c=0.01$, and $w_a=0.1$, with outer scale
$\alpha_{\mathrm{KKT}}=1$; proxy dual bounds are
$\lambda_{\max}=\nu_{\max}=100$.

The continuation-value model is trained from the expert trajectories by
regressing to the discounted return-to-go. For an expert state
$x_i^\paren{d}$, we use the target
\begin{align}
    G_i^\paren{d}
    =
    \sum_{j=i}^{T}
    \gamma^{j-i}
    r\!\left(x_j^\paren{d},u_j^\paren{d}\right),
\end{align}
and minimize
\begin{align}
    \loss_V
    =
    \E_{\D_E}
    \left[
    \left(
    \hat V(x_i^\paren{d})-G_i^\paren{d}
    \right)^2
    \right].
\end{align}
The resulting $\hat V$ is used only to score the terminal states of
counterfactual and expert-action rollouts in the value-augmented advantage.

All results are averaged over the same set of $N=10$ random seeds. Experiments
are conducted on a cloud-based Linux compute node with 32 CPU cores
(Intel Xeon Gold 6426Y, 2.5 GHz), 64 GB system memory, and one NVIDIA L40S GPU
with 46 GB VRAM. One seed with 20,000 total training steps takes approximately
5--10 minutes. Counterfactual generation uses $B=256$ sampled expert windows,
$K=64$ perturbed action sequences, perturbation scale $\sigma=2$ (for MuJoCo), $\sigma = 0.8$ (for 2D maze), and rollout
horizon $H=24$. The
counterfactual margin is $m=2$ in the 2D maze experiments and $m=1$ in the
MuJoCo experiments, with $\alpha_{\mathrm{CF}}=1$.


In MuJoCo environments, the downstream policy extraction involves using suboptimal data for Advantage Weighted Regression. We hereby retain this evaluation protocol consistent with ICSDICE and SMODICE because they explicitly \textit{require} a suboptimal dataset to function as a constraint learner. We clarify that the additional suboptimal data never enters the constraint learning phase for CF-KKT, and is used solely for evaluation.

\subsection{Maze Planning Visualizations}
\label{app:maze-planning-vis}
Figure~\ref{fig:heatmap} shows representative downstream plans produced from the
learned constraint masks.
\begin{figure}[t]
    \centering
    \includegraphics[width=0.9\linewidth]{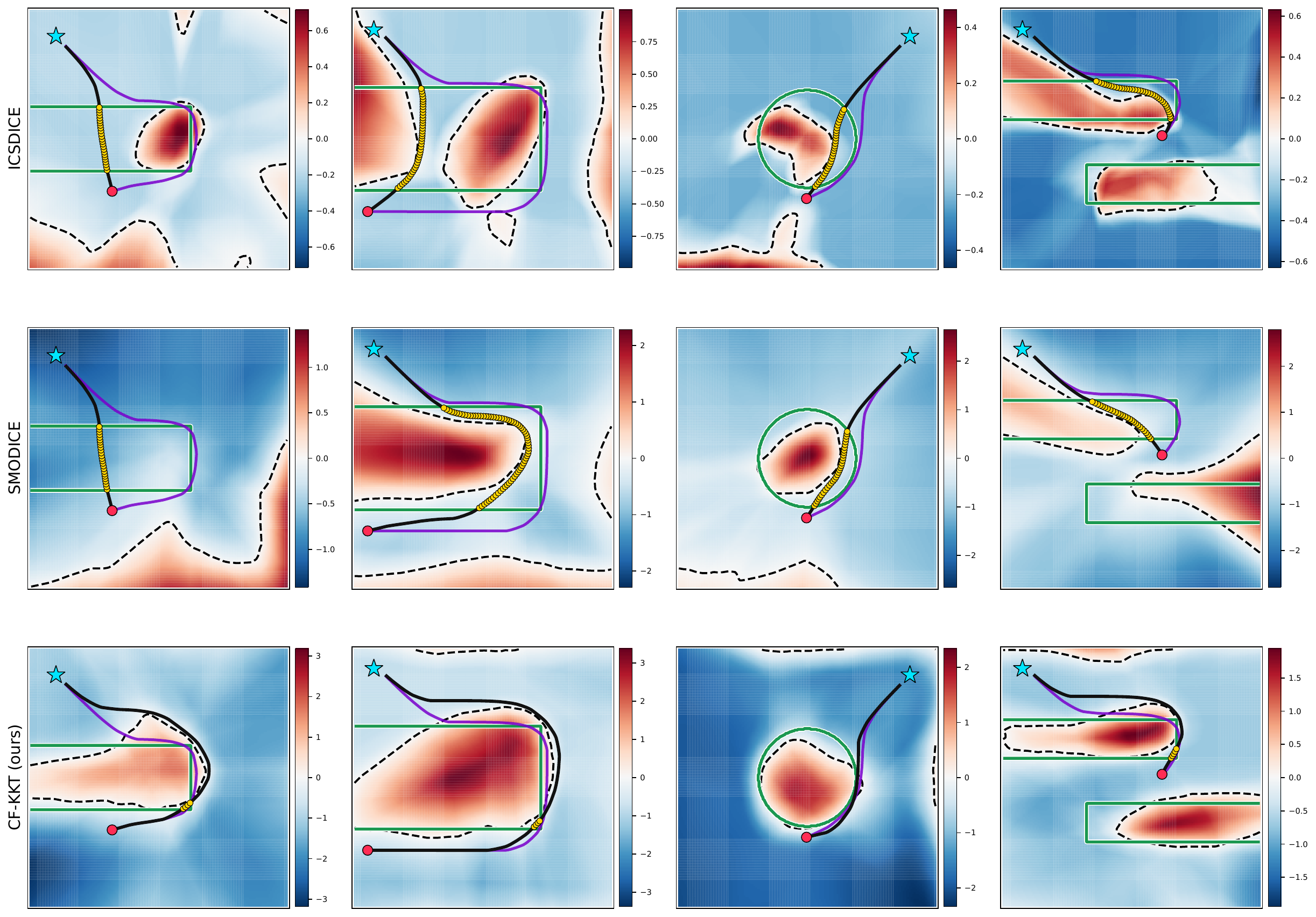}
    \caption{Representative maze plans. The dashed contour is the predicted unsafe boundary, green the true unsafe region, purple the oracle route, and yellow the executed timesteps inside the true unsafe region. The goal is indicated by blue star.}
    \label{fig:heatmap}
\end{figure}

\subsection{Robustness to Learned Reward Models}
\label{subsec:exp-reward-model}
Our main method assumes access to the analytical task reward, as is standard in
inverse optimal control, because KKT stationarity depends explicitly on
$\nabla_{x,u}r(x,u)$. We evaluate how sensitive CF-KKT is to relaxing this
assumption by replacing $r$ with a neural approximation
$\hat r(\cdot,\cdot;\omega)$.

A reward model trained only to match scalar reward values need not recover the
correct local reward geometry: small prediction error in $\hat r$ does not
imply that $\nabla_{x,u}\hat r$ matches $\nabla_{x,u}r$, even though the latter
directly enters the KKT stationarity residual. We therefore compare two reward
models. \textbf{MLP} minimizes only reward-value prediction error, while
\textbf{MLP + reg.} uses the same architecture and additionally penalizes
reward-gradient mismatch:
\begin{align}
\label{Eqn: Reward Gradient Regularization}
    \loss_r(\omega)
    =
    \E\Big[
    &\big(
    \hat r(x_t,u_t;\omega)-r(x_t,u_t)
    \big)^2+
    w_{\nabla r}
    \left\|
    \nabla_{x_t,u_t}\hat r(x_t,u_t;\omega)
    -
    \nabla_{x_t,u_t}r(x_t,u_t)
    \right\|_2^2
    \Big].
\end{align}
Here $w_{\nabla r}=0$ recovers the plain MLP, whereas
$w_{\nabla r}>0$ explicitly regularizes the learned reward gradient toward the
analytical one. In this ablation, $\hat r$ replaces $r$ in the KKT stationarity
term and in the counterfactual return calculation.

Table~\ref{tab:learned_reward} shows that gradient regularization reduces
constraint violations on all five environments, with especially large
reductions on \texttt{HalfCheetahWithObstacle} ($133.5\!\to\!51.1$) and
\texttt{BlockedHopper} ($65.1\!\to\!37.3$). These results indicate that CF-KKT
can remain effective when the analytical reward is replaced by a learned
approximation, provided that the local reward-gradient structure used by the
KKT conditions is also preserved.

\begin{table}[t]
\centering
\caption{Learned-reward robustness (mean $\pm$ std, $N=10$); MLP + reg. also matches analytical reward gradients.}
\label{tab:learned_reward}
\scriptsize
\setlength{\tabcolsep}{3pt}
\renewcommand{\arraystretch}{0.9}
\begin{tabular}{lcccc}
\toprule
& \multicolumn{2}{c}{Violations $\downarrow$}
& \multicolumn{2}{c}{Return $\uparrow$} \\
\cmidrule(lr){2-3}\cmidrule(lr){4-5}
Environment & MLP & MLP + reg. & MLP & MLP + reg. \\
\midrule
Ant\_ls                 & $7.1 \pm 13.8$   & $2.3 \pm 0.5$    & $1177 \pm 98$  & $1177 \pm 68$  \\
BlockedAnt          & $89.4 \pm 40.2$  & $67.0 \pm 19.4$  & $2142 \pm 621$ & $2165 \pm 552$ \\
BlockedHopper       & $65.1 \pm 26.9$  & $37.3 \pm 30.3$  & $342 \pm 239$  & $277 \pm 102$  \\
HalfCheetah (obst.) & $133.5 \pm 32.7$ & $51.1 \pm 33.9$  & $3126 \pm 407$ & $2938 \pm 408$ \\
Walker\_ls              & $138.0 \pm 76.4$ & $100.2 \pm 76.5$ & $676 \pm 420$  & $447 \pm 309$  \\
\bottomrule
\end{tabular}
\end{table}

\subsection{Additional Benchmarking Against Online ICRL Methods.}
Our main experiments focus on offline constraint inference. As an auxiliary
comparison, we also evaluate the \emph{KKT-informed component} of our framework
within the online ICRL benchmark of~\cite{liu2023benchmarking}. This experiment
does not use the counterfactual component of CF-KKT and should therefore be
interpreted as a diagnostic of the KKT-based constraint signal rather than an
evaluation of the full method.

We compare this KKT-only constraint learner with BC2L, MECL, and GACL in the
\texttt{BlockedHalfCheetah} environment. The online baselines receive
additional interaction data from the environment, whereas the KKT-only
constraint learner relies on the initial expert demonstrations. Figure
\ref{fig:online-icrl} reports downstream policy performance when each method is
used as the upstream constraint-learning component. The KKT-only variant
achieves lower constraint violation and higher feasible reward than BC2L and
MECL, while remaining comparable to GACL, and exhibits lower variance than
BC2L and MECL.

\begin{figure}[t]
    \centering
    \includegraphics[width=0.48\linewidth]{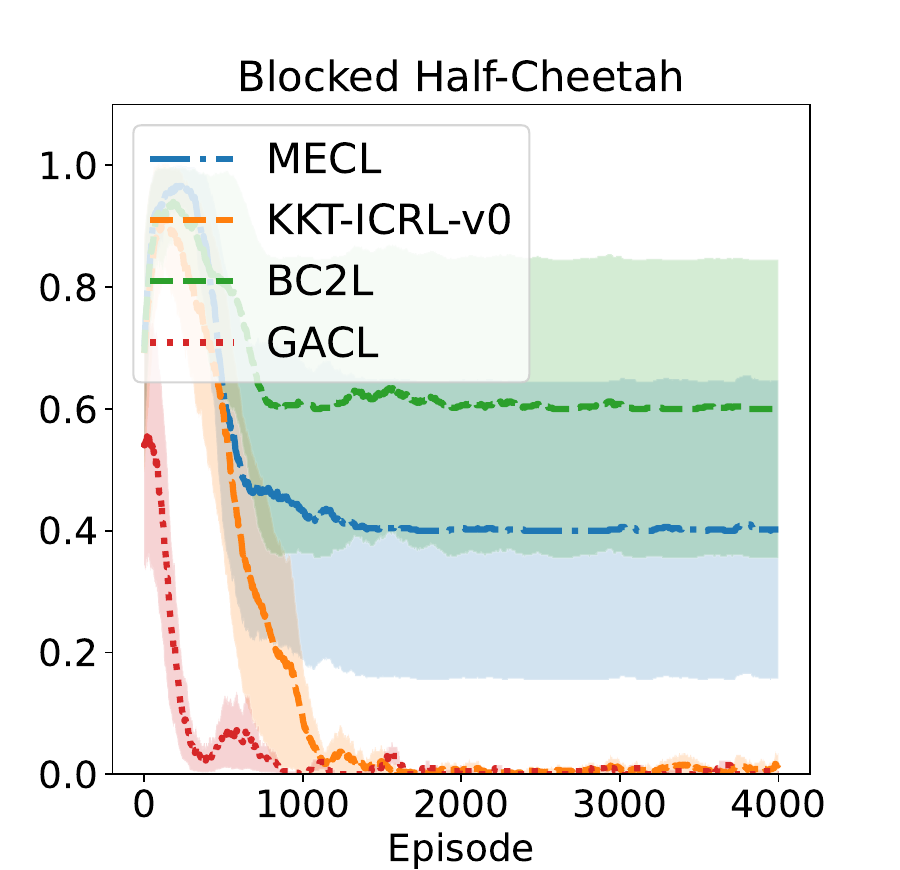}
    \hfill
    \includegraphics[width=0.48\linewidth]{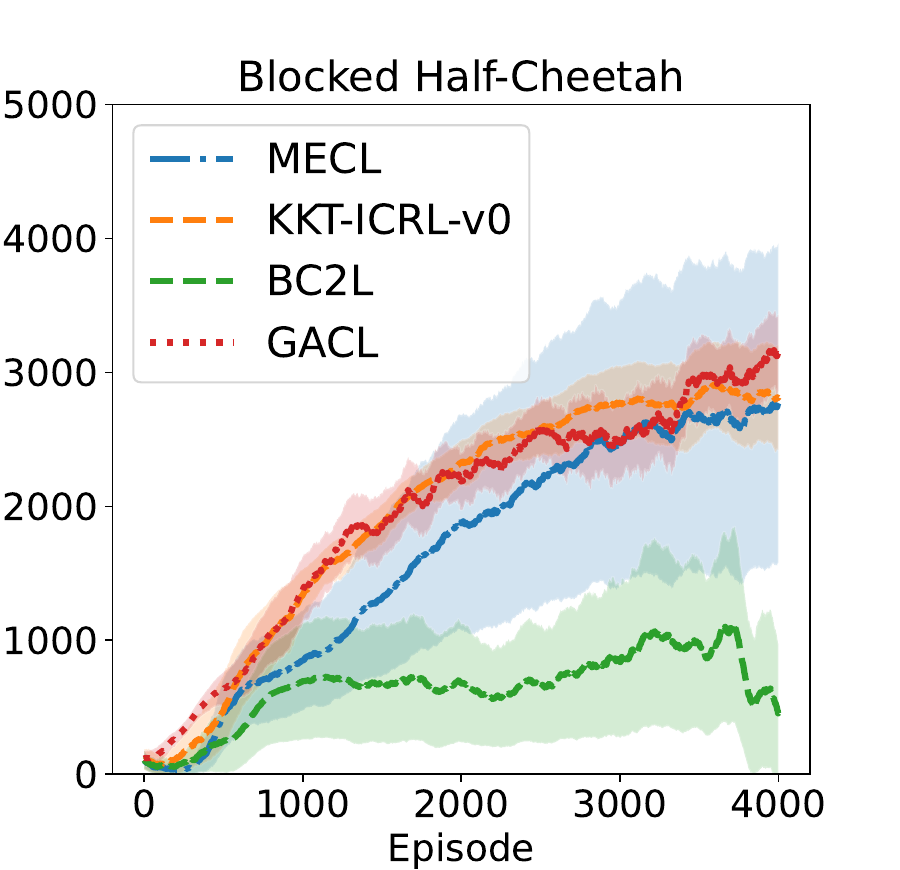}
    \caption{Constraint violation rate (left) and feasible cumulative reward
    (right) for the auxiliary online ICRL comparison.}
    \label{fig:online-icrl}
\end{figure}

\subsection{Method Sweep on Counterfactual generation}
We separately study the two design choices used in counterfactual generation:
the rollout horizon $H$ and the counterfactual acceptance rule. First, we vary
$H\in\{12,24,48\}$ while keeping the positive-advantage gate
$w_k=\mathbf{1}[\hat A_V^{(k)}>0]$ fixed, where $H=24$ corresponds to the
default CF-KKT setting. Second, fixing $H=24$, we compare this gate against a
quantile-based rule that retains the top $25\%$ of counterfactuals ranked by
$\hat A_V^{(k)}$.

\begin{table}[t]
\centering
\tiny
\setlength{\tabcolsep}{2.5pt}
\renewcommand{\arraystretch}{0.80}
\resizebox{\linewidth}{!}{%
\begin{tabular}{l cc cc cc}
\toprule
& \multicolumn{2}{c}{$H=12$}
& \multicolumn{2}{c}{$H=48$}
& \multicolumn{2}{c}{CF-KKT ($H=24$)} \\
\cmidrule(lr){2-3}
\cmidrule(lr){4-5}
\cmidrule(lr){6-7}
Env.
& Viol. $\downarrow$ & Ret. $\uparrow$
& Viol. $\downarrow$ & Ret. $\uparrow$
& Viol. $\downarrow$ & Ret. $\uparrow$ \\
\midrule
Ant\_ls
& $3.4\!\pm\!1.5$ & $1157\!\pm\!57$
& $3.3\!\pm\!1.6$ & $\mathbf{1167\!\pm\!84}$
& $\mathbf{2.8\!\pm\!0.9}$ & $1130\!\pm\!91$ \\

BlockedAnt
& $16.5\!\pm\!3.3$ & $2136\!\pm\!635$
& $18.5\!\pm\!4.9$ & $1968\!\pm\!733$
& $\mathbf{15.3\!\pm\!4.2}$ & $\mathbf{2449\!\pm\!237}$ \\

BlockedHopper
& $64.9\!\pm\!83.5$ & $317\!\pm\!139$
& $28.6\!\pm\!32.7$ & $\mathbf{510\!\pm\!433}$
& $\mathbf{17.6\!\pm\!21.2}$ & $454\!\pm\!252$ \\

HalfCheetah
& $93.9\!\pm\!67.5$ & $2196\!\pm\!583$
& $80.2\!\pm\!78.7$ & $\mathbf{2605\!\pm\!621}$
& $\mathbf{54.3\!\pm\!48.9}$ & $2338\!\pm\!788$ \\

Walker\_ls
& $85.4\!\pm\!43.1$ & $\mathbf{413\!\pm\!213}$
& $89.6\!\pm\!65.6$ & $406\!\pm\!279$
& $\mathbf{74.3\!\pm\!56.3}$ & $410\!\pm\!283$ \\
\bottomrule
\end{tabular}%
}
\vspace{-4pt}
\caption{Effect of counterfactual rollout horizon with the
positive-advantage gate fixed. CF-KKT uses the default $H=24$.}
\label{tab:horizon-sweep}
\vspace{-6pt}
\end{table}

\begin{table}[t]
\centering
\scriptsize
\setlength{\tabcolsep}{3.5pt}
\renewcommand{\arraystretch}{0.82}
\begin{tabular}{l cc cc}
\toprule
& \multicolumn{2}{c}{$q=0.25$}
& \multicolumn{2}{c}{CF-KKT} \\
\cmidrule(lr){2-3}
\cmidrule(lr){4-5}
Env.
& Viol. $\downarrow$ & Ret. $\uparrow$
& Viol. $\downarrow$ & Ret. $\uparrow$ \\
\midrule
Ant\_ls
& $3.3\!\pm\!1.7$ & $1117\!\pm\!102$
& $\mathbf{2.8\!\pm\!0.9}$ & $\mathbf{1130\!\pm\!91}$ \\

BlockedAnt
& $21.8\!\pm\!14.4$ & $2096\!\pm\!604$
& $\mathbf{15.3\!\pm\!4.2}$ & $\mathbf{2449\!\pm\!237}$ \\

BlockedHopper
& $27.3\!\pm\!32.3$ & $\mathbf{624\!\pm\!261}$
& $\mathbf{17.6\!\pm\!21.2}$ & $454\!\pm\!252$ \\

HalfCheetah
& $84.6\!\pm\!58.2$ & $\mathbf{2735\!\pm\!494}$
& $\mathbf{54.3\!\pm\!48.9}$ & $2338\!\pm\!788$ \\

Walker\_ls
& $97.3\!\pm\!92.3$ & $\mathbf{422\!\pm\!393}$
& $\mathbf{74.3\!\pm\!56.3}$ & $410\!\pm\!283$ \\
\bottomrule
\end{tabular}
\vspace{-4pt}
\caption{Effect of the counterfactual acceptance rule at $H=24$.
The quantile rule retains the top $25\%$ of candidates by
$\hat A_V^{(k)}$, whereas CF-KKT retains all candidates with
$\hat A_V^{(k)}>0$.}
\label{tab:acceptance-sweep}
\vspace{-6pt}
\end{table}

The horizon sweep in Table~\ref{tab:horizon-sweep} shows that CF-KKT is not highly sensitive to the precise
rollout length: $H=12$, $24$, and $48$ all yield comparable downstream
performance, while the default $H=24$ achieves the lowest mean violation count
on all five environments. The acceptance-rule comparison in Table~\ref{tab:acceptance-sweep} shows CF-KKT's positive-advantage-based gating achieves lower violation than the quantile-based rule. 
\section{Empirical Validation of the Counterfactual Premise}
\label{app:cf-premise}

CF-KKT retains a generated counterfactual when its learned-model
value-augmented advantage is positive, i.e.,
$\hat A_V^\paren{k}>0$. Lemma~\ref{lem:counterfactual-infeasibility}
establishes that, under global expert optimality, exact dynamics, and the exact
constrained continuation value, a positive value-augmented advantage implies
that the counterfactual cannot remain feasible throughout the rollout. In
practice, CF-KKT evaluates this criterion using the learned dynamics
$\hat f$ and continuation-value estimate $\hat V$. Here we empirically assess
how reliably the resulting practical selection rule identifies true
constraint violations. Ground-truth constraint labels are used exclusively
for this diagnostic and are never used to train CF-KKT.

\paragraph{Metrics.}
For each selected counterfactual rollout
$\{x_h^\paren{k}\}_{h=1}^{H}$, we evaluate the benchmark ground-truth
state constraint $c^\star$ and report:
\begin{itemize}
    \item \textbf{Bag precision:}
    $\Pr[\exists h:\,c^\star(x_h^\paren{k})>0
    \mid \hat A_V^\paren{k}>0]$, the fraction of selected rollouts containing
    at least one truly infeasible state. This is the metric most directly
    aligned with the trajectory-level implication in
    Lemma~\ref{lem:counterfactual-infeasibility}.

    \item \textbf{State precision:}
    $\Pr[c^\star(x_h^\paren{k})>0
    \mid \hat A_V^\paren{k}>0]$, evaluated over all states in selected
    rollouts. Unlike bag precision, this metric measures the density of
    infeasible states within a selected rollout; the lemma does not require
    every state to be infeasible.

    \item \textbf{Enrichment:} the state-level violation rate among selected
    counterfactual rollouts divided by the violation rate in the benchmark
    suboptimal dataset $\D_\pi$. This metric measures whether the generator
    preferentially samples infeasible behavior relative to the benchmark
    background distribution. Here $\D_\pi$ serves only as the reference
    distribution for this metric.

    \item \textbf{Per-step violation rate:} the violation frequency at the
    first and final rollout states, $h=1$ and $h=H$, indicating how
    infeasibility develops over the counterfactual horizon.
\end{itemize}

\begin{table}[h]
\centering
\small
\begin{tabular}{lccccc}
\toprule
\textbf{Environment}
 & \textbf{Bag prec.} & \textbf{State prec.} & \textbf{Enrich.}
 & $h{=}1$ & $h{=}H$ \\
\midrule
Ant\_ls        & 1.000 & 0.848 & 3.75 & 0.87 & 0.81 \\
BlockedAnt     & 1.000 & 1.000 & 2.72 & 1.00 & 1.00 \\
BlockedHopper  & 0.922 & 0.749 & 1.37 & 0.08 & 0.72 \\
Walker\_ls     & 1.000 & 0.969 & 1.31 & 1.00 & 0.95 \\
HalfCheetah    & 0.953 & 0.867 & 1.78 & 0.00 & 0.83 \\
\bottomrule
\end{tabular}
\caption{Ground-truth diagnostics of selected CF-KKT counterfactuals over
$N=10$ seeds. Metrics are defined in the text.}
\label{tab:cf-premise}
\end{table}

\paragraph{Positive value-augmented advantage consistently identifies infeasible rollouts.}
Bag precision is at least $0.922$ and enrichment exceeds $1$ on all five
environments. Thus, the practical learned-model selection rule reliably
identifies trajectory-level infeasibility and concentrates selected rollouts
in regions with more violations than the benchmark suboptimal distribution.

\paragraph{Constraint violations may emerge only after several model steps.}
This effect is most apparent on \texttt{BlockedHopper} and
\texttt{HalfCheetah}. In \texttt{BlockedHopper}, the constraint bounds a joint
angle, so action perturbations influence the constrained quantity progressively
through the dynamics; the violation rate increases from $0.08$ at $h=1$ to
$0.72$ at $h=H$. In \texttt{HalfCheetah}, the constraint depends on absolute
horizontal position, which responds indirectly through the evolution of
velocity. Consequently, no selected counterfactual violates the constraint at
$h=1$, whereas the violation rate reaches $0.83$ by $h=H$. These results
support the use of multi-step rollouts for constraints whose effects are
dynamically delayed.

\paragraph{State-level density supports the mean-based surrogate.}
State precision ranges from $0.749$ to $1.000$ across the five tasks, showing
that violations are typically not isolated to a single state within a selected
rollout. This empirical density supports the mean constraint value in the
counterfactual loss as a dense, smooth surrogate, even though the theoretical
guarantee requires only the existence of a violation somewhere along the
trajectory.

\section{Ablation of Supervision Signals}
\label{app:ablation}

\subsection{Experimental Setting}

CF-KKT combines expert-side KKT supervision with synthetic counterfactual
supervision. Because these MuJoCo benchmarks use state-only constraints, we
write $\hat c_\theta(x)$ for brevity. To isolate the contribution of each term,
we decompose the expert-side objective into
\begin{align}
\mathcal{L}_{\mathrm{stat}}
&=
\E_{\D_E}\!\left[
\left\|\nabla_{\xi}\hat L(\xi)\right\|_2^2
\right],\\
\mathcal{L}_{\mathrm{prim}}
&=
\E_{\D_E}\!\left[
\max\{\hat c_\theta,0\}^2
\right],\\
\mathcal{L}_{\mathrm{comp}}
&=
\E_{\D_E}\!\left[
(\lambda \hat c_\theta)^2
\right],\\
\mathcal{L}_{\mathrm{act}}
&=
\E_{\D_E}\!\left[
|\hat c_\theta|
\right].
\end{align}
The complete KKT-informed loss is
\begin{align}
\loss_\KKT
=
\mathcal{L}_{\mathrm{stat}}
+
w_p\mathcal{L}_{\mathrm{prim}}
+
w_c\mathcal{L}_{\mathrm{comp}}
+
w_a\mathcal{L}_{\mathrm{act}},
\end{align}
with $w_p=10$, $w_c=0.01$, and $w_a=0.1$. We further define a generic classifier that does not use KKT optimality:
\begin{align}
\mathcal{L}_{\mathrm{safe}}
=
\E_{\D_E}\!\left[
w_p\max\{\hat c_\theta,0\}^2
+
w_a|\hat c_\theta|
\right],
\end{align}
and denote the counterfactual loss by
\begin{align}
\loss_{\mathrm{CF}}
=
\E_b\!\left[
\max\!\left\{
m-\frac{1}{H}\sum_h \hat c_\theta(x_{b,h}),\,0
\right\}^{2}
\right],
\qquad m=1.
\end{align}

We compare the following seven controlled conditions:
\begin{itemize}
    \item \textbf{$A$:}
    $\mathcal{L}_{\mathrm{safe}}+\loss_{\mathrm{CF}}$.
    This retains expert-side feasibility/activity supervision and
    counterfactual supervision, but removes the optimality residual from Lagrangians.

    \item \textbf{$B$:}
    $A+w_c\mathcal{L}_{\mathrm{comp}}$.
    This isolates the additional contribution of complementary slackness.

    \item \textbf{$A^{-\mathrm{act}}$:}
    $A$ with $w_a=0$.
    This isolates the effect of the activity term.

    \item \textbf{$A^{-\mathrm{cf}}$:}
    $\mathcal{L}_{\mathrm{safe}}$ only.
    This removes counterfactual supervision.

    \item \textbf{CF-only:}
    $\loss_{\mathrm{CF}}$ only.
    This removes all supervision acting directly on expert demonstrations.

    \item \textbf{KKT-only:}
    $\loss_\KKT$ only, obtained by setting $\alpha_{\mathrm{CF}}=0$. This
    condition therefore isolates what the optimality conditions alone identify,
    and serves as the counterpart of CF-only.

    \item \textbf{CF-KKT:}
    the complete objective
    $\alpha_\KKT\loss_\KKT+\alpha_{\mathrm{CF}}\loss_{\mathrm{CF}}$.
\end{itemize}
All experimental settings other than the indicated objective terms are held
fixed across conditions.

\subsection{Metrics}

We evaluate the learned constraint directly using two complementary metrics.
First, the mask admission rate
\begin{align}
\rho_\theta
=
\E_{\D_\pi}\!\left[
\mathbf{1}\!\left[\hat c_\theta(x')\leq0\right]
\right]
\end{align}
is the fraction of transitions in the benchmark suboptimal dataset admitted by
the learned feasibility mask. We compare this quantity with $\rho^\star$, the
fraction admitted by the ground-truth constraint. Values substantially below
$\rho^\star$ indicate an overly restrictive mask, while values substantially
above $\rho^\star$ indicate an overly permissive one.

Second, we report ROC AUC of $\hat c_\theta$ against ground-truth constraint
labels on held-out suboptimal transitions. AUC is threshold-free and therefore
measures whether the learned score ranks feasible and infeasible states
correctly independently of the operating threshold. Ground-truth labels are
used only for evaluation.

\subsection{Results}

\begin{table}[h]
\centering
\small
\resizebox{\linewidth}{!}{%
\begin{tabular}{lcccccccc}
\toprule
\textbf{Environment}
& $\rho^\star$
& $A$
& $B$
& $A^{-\mathrm{act}}$
& $A^{-\mathrm{cf}}$
& \textbf{CF-only}
& \textbf{KKT-only}
& \textbf{CF-KKT} \\
\midrule
Ant\_ls        & 0.774 & 0.023 & 0.014 & 0.095 & 0.426 & 0.000 & 0.992 & 0.681 \\
BlockedAnt     & 0.632 & 0.151 & 0.149 & 0.281 & 0.480 & 0.000 & 0.890 & 0.528 \\
BlockedHopper  & 0.453 & 0.088 & 0.080 & 0.445 & 0.471 & 0.000 & 0.586 & 0.241 \\
HalfCheetah    & 0.512 & 0.009 & 0.003 & 0.064 & 0.431 & 0.000 & 0.394 & 0.360 \\
Walker\_ls     & 0.259 & 0.043 & 0.051 & 0.208 & 0.460 & 0.000 & 0.692 & 0.203 \\
\bottomrule
\end{tabular}}
\caption{Mask admission rates across the seven ablation conditions.}
\label{tab:ablation-rho}
\end{table}

\begin{table}[h]
\centering
\small
\resizebox{\linewidth}{!}{%
\begin{tabular}{lccccccc}
\toprule
\textbf{Environment}
& $A$
& $B$
& $A^{-\mathrm{act}}$
& $A^{-\mathrm{cf}}$
& \textbf{CF-only}
& \textbf{KKT-only}
& \textbf{CF-KKT} \\
\midrule
Ant\_ls        & 0.768 & 0.746 & 0.804 & 0.570 & \textbf{0.907} & 0.458 & 0.831 \\
BlockedAnt     & 0.444 & 0.438 & 0.446 & 0.449 & \textbf{0.904} & 0.522 & 0.616 \\
BlockedHopper  & 0.614 & 0.617 & 0.574 & 0.510 & \textbf{0.802} & 0.481 & 0.602 \\
HalfCheetah    & 0.635 & 0.619 & 0.636 & 0.494 & 0.655 & 0.976 & \textbf{0.993} \\
Walker\_ls     & 0.779 & \textbf{0.787} & 0.783 & 0.524 & 0.624 & 0.382 & 0.705 \\
\bottomrule
\end{tabular}}
\caption{ROC AUC of the learned constraint score on held-out transitions.}
\label{tab:ablation-auc}
\end{table}

\subsection{Analysis}

Across the five environments, CF-only and full CF-KKT attain the two highest
mean AUC values, $0.778$ and $0.749$, respectively, whereas KKT-only attains
$0.564$. We therefore compare CF-only and CF-KKT downstream in
Table~\ref{tab:cfonly-vs-cfkkt}, while using the KKT-only condition to isolate
what can be identified from expert-side optimality supervision alone.

\begin{table}[t]
\centering
\small
\setlength{\tabcolsep}{5pt}
\begin{tabular}{l cc cc}
\toprule
& \multicolumn{2}{c}{\textbf{CF-only}}
& \multicolumn{2}{c}{\textbf{CF-KKT}} \\
\cmidrule(lr){2-3} \cmidrule(lr){4-5}
\textbf{Environment}
& \textbf{Viol.} $\downarrow$ & \textbf{Return} $\uparrow$
& \textbf{Viol.} $\downarrow$ & \textbf{Return} $\uparrow$ \\
\midrule
Ant\_ls       & $16.6 \pm 10.4$  & $246 \pm 176^{\dagger}$
              & $\mathbf{2.8 \pm 0.9}$  & $\mathbf{1130 \pm 91}$ \\
BlockedAnt    & $163.4 \pm 137.9$ & $276 \pm 182^{\dagger}$
              & $\mathbf{15.3 \pm 4.2}$ & $\mathbf{2449 \pm 237}$ \\
BlockedHopper & $2.1 \pm 6.1^{\dagger}$ & $112 \pm 298^{\dagger}$
              & $17.6 \pm 21.2$ & $\mathbf{454 \pm 252}$ \\
HalfCheetah   & $0.0 \pm 0.0^{\dagger}$ & $-3 \pm 3^{\dagger}$
              & $54.3 \pm 48.9$ & $\mathbf{2338 \pm 788}$ \\
Walker\_ls    & $0.0 \pm 0.0^{\dagger}$ & $-6 \pm 9^{\dagger}$
              & $74.3 \pm 56.3$ & $\mathbf{410 \pm 283}$ \\
\bottomrule
\end{tabular}
\caption{Downstream comparison of CF-only and CF-KKT. Mean $\pm$ standard
deviation over $N=10$ seeds; $^{\dagger}$ denotes severe return collapse.}
\label{tab:cfonly-vs-cfkkt}
\end{table}

\paragraph{Counterfactual supervision provides off-demonstration discrimination,
but does not calibrate the feasibility threshold by itself.}
CF-only achieves the highest mean AUC, indicating that counterfactuals provide
a strong signal for ranking feasible against infeasible states. However,
$\rho_\theta=0$ on all five environments, so the fixed feasibility threshold
rejects every transition. The comparison between $A$ and $A^{-\mathrm{cf}}$
shows the same effect: adding counterfactual supervision improves AUC on four
of five tasks, but also makes the learned mask substantially more restrictive, $\rho_{\theta}$ of A is close to 0 on 4/5 envs, meaning that it rejects almost every transition not from the expert.
Thus, counterfactual supervision provides useful off-demonstration separation,
but by itself does not determine where the feasibility boundary should lie.

\paragraph{KKT supervision provides complementary expert-side information, but
alone does not reliably identify the unsafe region.}
KKT-only exhibits nearly the opposite failure mode from CF-only. Its admission
rate exceeds $\rho^\star$ on four of five environments, indicating that
expert-side optimality supervision alone generally produces an overly
permissive mask. Its AUC is also near chance on
\texttt{Ant\_ls}, \texttt{BlockedAnt}, \texttt{BlockedHopper}, and
\texttt{Walker\_ls}, showing that supervision evaluated only on expert
transitions provides limited information about how unseen states should be
ranked. \texttt{HalfCheetah} is a notable exception, where KKT-only already
achieves AUC $0.976$, suggesting that the local optimality structure is
particularly informative in that environment.

\paragraph{The two supervision sources correct complementary failure modes.}
Combining KKT and counterfactual supervision substantially improves calibration:
the mean absolute deviation
$\frac{1}{5}\sum_e |\rho_{\theta,e}-\rho^\star_e|$
is $0.526$ for CF-only and $0.232$ for KKT-only, but decreases to $0.123$ for
full CF-KKT. The finer expert-side ablations support the same interpretation.
Removing the activity term, $A^{-\mathrm{act}}$, increases $\rho_\theta$ in
every environment while changing AUC only modestly, indicating that this term
primarily affects threshold placement. Conditions $A$ and $B$ are nearly
identical, suggesting that complementary slackness contributes little in the
absence of the stationarity residual. Full CF-KKT combines the
off-demonstration discrimination supplied by counterfactuals with the
expert-side structure supplied by the KKT-informed objective, yielding both
strong ranking performance and a substantially better-calibrated feasibility
mask.

\paragraph{Downstream control requires both discrimination and a useful
operating threshold.}
Table~\ref{tab:cfonly-vs-cfkkt} illustrates why AUC alone is insufficient.
Despite its high mean AUC, CF-only suffers severe return collapse on every
task; its low violation counts on \texttt{BlockedHopper},
\texttt{HalfCheetah}, and \texttt{Walker\_ls} therefore arise largely because
the learned mask prevents meaningful task progress. CF-KKT instead attains
substantially higher return on all five environments and improves both return
and violation count on \texttt{Ant\_ls} and \texttt{BlockedAnt}. Together,
these results support the intended division of roles: counterfactual
supervision identifies off-demonstration infeasibility, while KKT-informed
expert supervision helps place that discrimination at a practically useful
feasibility threshold.
\section{Proofs}
\label{app:proofs}

\subsection{Sensitivity Analysis for MIP-based Inference of Constraint Parameters}
\label{subsec: Sensitivity Analysis for MIP-based Inference of Constraint Parameters}


Here, we analyze the sensitivity of the MIP-based constraint inference method in Sec. \ref{subsec: MIP Formulation for Parametric Constraints} against dynamics misspecification under a set of mild regularity conditions.
Let $\boldeta_{1:D} := (\eta_t^\paren{d}: t \in [0,T], d \in [D_E]) \in \R^{TD_E}$, 
$\boldrho_{1:D} := (\rho_t^\paren{d}: t \in [0,T], d \in [D_E]) \in \R^{TD_E}$, and 
$\boldkappa_{1:D_E} := (\kappa_t^\paren{d}: t \in [0,T], d \in [D_E]) \in \R^{TD_E}$ denote dual variables for \eqref{Eqn: Inverse KKT, Relaxed} that respectively correspond to the constraints
$\hat c(x_t^\paren{d}, u_t^\paren{d}; \theta) \leq 0$, $\lambda_t^\paren{d} \geq 0$, and $\lambda_t^\paren{d} \hat c(x_t^\paren{d}, u_t^\paren{d}; \theta) = 0$, across $d \in [D_E]$ and $t \in [0,T]$. 
Then \eqref{Eqn: Inverse KKT, Relaxed} admits a Lagrangian map given by:
\begin{align}
\label{Eqn: Lagrangian, Inverse KKT, Relaxed}
    &\lagrangian_\invRel(\theta, \boldlambda_{1:D_E}, \boldnu_{1:D_E}, \boldeta_{1:D_E}, \boldrho_{1:D_E}, \boldkappa_{1:D_E}) \\ \nonumber
    := \ &\textstyle\sum_{d=1}^{D_E} \Big[ \big\Vert \hat \bolds(\boldx^\paren{d}, \boldu^\paren{d}, \boldlambda^\paren{d}, \boldnu^\paren{d}; \theta, \psi) \big\Vert_2^2 +  \sum_{t=0}^T \eta_t^\paren{d} \hat c(x_t^\paren{d}, u_t^\paren{d}; \theta) \\ \nonumber
    &\hspace{1.5cm} - \textstyle\sum_{t=0}^T \rho_t^\paren{d} \lambda_t^\paren{d} + \textstyle\sum_{t=0}^T \kappa_t^\paren{d} \lambda_t^\paren{d} \hat c(x_t^\paren{d}, u_t^\paren{d}; \theta) \Big],
\end{align}
and KKT conditions given by:
\begin{subequations}
\label{Eqn: KKT for Relaxed Inverse KKT}
\begin{align}
\label{Eqn: KKT for Relaxed Inverse KKT, constraint and dual of primal}
    \hat c(x_t^\paren{d}, u_t^\paren{d}; \theta) \leq 0 \quad &\perp \quad \lambda_t^\paren{d} \geq 0, \quad \forall t \in [T], \ \forall d \in [D_E], \\
\label{Eqn: KKT for Relaxed Inverse KKT, constraint and dual of dual}
    \hat c(x_t^\paren{d}, u_t^\paren{d}; \theta) \leq 0 \quad &\perp \quad \eta_t^\paren{d} \geq 0, \quad \forall t \in [T], \ \forall d \in [D_E], \\
\label{Eqn: KKT for Relaxed Inverse KKT, dual of primal and dual of dual}
    - \lambda_t^\paren{d} \leq 0 \quad &\perp \quad \rho_t^\paren{d} \geq 0, \quad \forall t \in [T], \ \forall d \in [D_E], \\
\label{Eqn: KKT for Relaxed Inverse KKT, Stationarity}
    \nabla_{(\theta, \boldlambda_{1:D}, \boldnu_{1:D})} \lagrangian_\invRel(\theta, &\boldlambda_{1:D}, \boldnu_{1:D}, \boldeta_{1:D}, \boldrho_{1:D}, \boldkappa_{1:D}) = \boldzero.
\end{align}
\end{subequations}

We say \eqref{Eqn: KKT for Relaxed Inverse KKT, constraint and dual of primal} holds with \textit{strict complementarity} if for each $t \in [0, T]$, one of the two inequalities $c(x_t, u_t; \theta) \leq 0$ and $\lambda_t \geq 0$ holds with strict inequality while the other holds with equality, and likewise for \eqref{Eqn: KKT for Relaxed Inverse KKT, constraint and dual of dual}-\eqref{Eqn: KKT for Relaxed Inverse KKT, dual of primal and dual of dual}. 
The following proposition uses the KKT system \eqref{Eqn: KKT for Relaxed Inverse KKT} and the Implicit Function Theorem to analyze the local sensitivity of the recovered constraint parameter $\theta$ to the neural network weights $\psi$ which characterize the learned dynamics map $\hat f(\cdot, \cdot; \psi)$. 
For any weights $\psi$, we will use $\boldz(\psi) := (\theta, \boldlambda_{1:D}, \boldnu_{1:D}, \boldeta_{1:D}, \boldrho_{1:D}, \boldkappa_{1:D})(\psi)$ to denote the concatenation of feasible primal and dual variables for \eqref{Eqn: Inverse KKT, Relaxed} at $\psi$.

\begin{proposition}
\label{Prop: Sensitivity Analysis for MIP-based Inference of Constraint Parameters}
Suppose 
$\hat f(\cdot, \cdot; \psi)$ is $C^4$ for each $\psi \in \R^{M_\psi}$.
Suppose also that, for a given NN weight $\hat \psi \in \R^{M_\psi}$, 
the relaxed inverse problem \eqref{Eqn: Inverse KKT, Relaxed} admits a primal-dual solution satisfying strict complementarity.
Define the map $F(\cdot)$ such that the equality $F(\boldz, \psi) = 0$, at $\psi = \hat \psi$, represents the KKT system \eqref{Eqn: KKT for Relaxed Inverse KKT} 
when all active inequalities in \eqref{Eqn: KKT for Relaxed Inverse KKT, constraint and dual of primal}-\eqref{Eqn: KKT for Relaxed Inverse KKT, dual of primal and dual of dual} are treated as equalities, and all inactive inequalities are ignored.
Suppose $F(\cdot)$ is $C^2$ in a neighborhood of $\hat\psi$, and $\nabla_\boldz F(\boldz^\star(\hat \psi), \hat\psi)$ is non-singular.
Then, there exists an open neighborhood $U$ of $\hat \psi$ and an open neighborhood $V$ of $\boldz^\star(\hat \psi)$ such that, for each $\psi \in U$, there exists a unique $\boldz^\star(\psi) \in V$ satisfying $F(\boldz^\star(\psi), \psi) = \emph{\boldzero}$. 
Moreover, for each $\psi \in U$:
\begin{align}
\label{Eqn: Sensitivity Analysis for MIP-based Inference of Constraint Parameters}
    \Vert \boldz^\star(\psi) - \boldz^\star(\hat\psi) \Vert_2 &\leq \Vert \nabla_\boldz F(\boldz^\star(\hat \psi), \hat\psi)^{-1} \nabla_\psi F(\boldz^\star(\hat \psi), \hat\psi) \Vert_2 \cdot \Vert \psi - \hat \psi \Vert_2  + \bigO(\Vert \psi - \hat \psi \Vert_2^2).
\end{align}
\end{proposition}

\begin{proof}
\vspace{-3pt}

By the Implicit Function Theorem \cite[Thm. 9.28]{Rudin1963PrinciplesOfMathematicalAnalysis}, there exist neighborhoods $U$ of $\hat{\psi}$ and $V$ of $\boldz^\star(\hat{\psi})$ such that, for each $\psi \in U$, $\boldz = \boldz^\star(\psi)$ is the unique solution in $V$ to $F(\boldz, \psi) = 0$, and $\boldz^\star(\cdot)$ is $C^2$ over $U$, with $\nabla_\psi \boldz^\star(\psi) = - \nabla_\boldz F(\boldz^\star(\psi), \psi)^{-1} \nabla_\psi F(\boldz^\star(\psi), \psi)$ for each $\psi \in U$.
By continuity and strict complementarity, $U$ and $V$ can be chosen sufficiently small such that the strict inequalities at $(\boldz^\star(\hat{\psi}),\hat{\psi})$ remain strict, so $F(\boldz^\star(\psi),\psi)=0$ continues to represent the KKT system \eqref{Eqn: KKT for Relaxed Inverse KKT} for each $\psi \in U$.
The bound \eqref{Eqn: Sensitivity Analysis for MIP-based Inference of Constraint Parameters} then follows by applying the Lagrange remainder theorem, followed by the triangle and Cauchy-Schwarz inequalities.
\end{proof}

\begin{remark}
Strict complementarity does not require all inequality constraints to be active, but merely excludes the degenerate case in which the primal and dual inequalities evaluated at an optimal primal-dual pair are simultaneously tight. 
The role of strict complementarity in Prop. \ref{Prop: Sensitivity Analysis for MIP-based Inference of Constraint Parameters} is to identify an unambiguous smooth equality representation of each complementarity condition at the nominal solution $\hat \psi$. 
Under strict complementarity, by continuity, the strict inequalities in the KKT system remain strict under sufficiently small perturbations of $\psi$, so the same equality representation remains valid, and the classical Implicit Function Theorem applies. 
Thus, Prop. \ref{Prop: Sensitivity Analysis for MIP-based Inference of Constraint Parameters} does not require the active set to remain fixed a priori or under arbitrary dynamics perturbations. 
More general sensitivity results that admit 
active-set transitions can be obtained by representing the KKT conditions as a generalized equation (e.g., formulated via an inclusion operator) under appropriate regularity conditions, although the resulting solution map need not be differentiable at such transitions.
\end{remark}

\vspace{-8pt}
Under the conditions in Prop. \ref{Prop: Sensitivity Analysis for MIP-based Inference of Constraint Parameters}, given any weights $\psi \in U$ for the learned dynamics map $\hat f(\cdot, \cdot; \psi)$, the optimal constraint parameter recovered by solving \eqref{Eqn: Inverse KKT, Relaxed} is given by the first $M_\theta$ components of $\boldz^\star(\psi)$.
Thus, the bound \eqref{Eqn: Sensitivity Analysis for MIP-based Inference of Constraint Parameters} quantifies the robustness of the recovered constraint parameter against fluctuations in the learned, dynamics-encoding weights $\psi$.

\end{document}